\documentclass{article}
\usepackage{ijcai20}

\usepackage{times}

\usepackage{soul}
\usepackage{url}
\usepackage[hidelinks]{hyperref}
\usepackage[utf8]{inputenc}
\usepackage[small]{caption}
\usepackage{graphicx}
\usepackage{amsmath}
\usepackage{booktabs}
\newcommand{\citet}[1]{\citeauthor{#1} \shortcite{#1}}
\newcommand{\citep}{\cite} 

\usepackage{fancyhdr}
\usepackage{color}
\usepackage{algorithm}
\usepackage{algorithmic}
\usepackage{eso-pic} 
\usepackage{forloop}

\usepackage{amsmath,amsthm,amscd,amsbsy,amssymb,latexsym,url,bm}

\usepackage{cuted}
\usepackage{flushend}

\usepackage{amsfonts}       
\usepackage{nicefrac}       
\usepackage{microtype}      

\usepackage{helvet}
\usepackage{courier}
\usepackage{multirow}
\usepackage{xcolor}
\usepackage{graphicx}
\usepackage{enumitem}
\usepackage{ragged2e}
\usepackage{epsfig}
\usepackage{dsfont}
\usepackage{amsthm}

\usepackage{subcaption}

\newcommand*\diff{\mathop{}\!\mathrm{d}}

\DeclareRobustCommand{\frac}[3][0pt]{{\begingroup\hspace{#1}#2\hspace{#1}\endgroup\over\hspace{#1}#3\hspace{#1}}}

\newtheorem{proposition}{Proposition}
\newtheorem{theorem}{Theorem}

\newtheorem{corollary}{Corollary}
\newtheorem{assumption}{Assumption}
\newtheorem{definition}{Definition}
\newtheorem*{remark}{Remark}
\newtheorem*{noproposition}{Proposition}

\newcommand{\bs}{\boldsymbol}

\title{Modelling Bounded Rationality in Multi-Agent Interactions \\by Generalized Recursive Reasoning}

\author{
Ying Wen$^{1,2}$\footnote{First two authors contribute equally.}\and
Yaodong Yang$^{1,2*}$\and
Jun Wang$^{1,2}$ \\ 
\affiliations
$^1$University College London\\
$^2$Huawei Research \& Development U.K.\\
\emails
\{ying.wen, yaodong.yang,  jun.wang\}@cs.ucl.ac.uk
}

\begin{document}

\maketitle

\begin{abstract}


Though limited in real-world decision making, most multi-agent reinforcement learning (MARL) models assume perfectly rational agents -- a property hardly met  due to individual's cognitive limitation and/or the tractability of the decision  problem. In this paper, we introduce generalized recursive reasoning (GR2) as a novel framework to model agents with different \emph{hierarchical} levels of rationality; our framework enables agents to exhibit varying levels of ``thinking'' ability thereby allowing higher-level agents  to best respond to various less sophisticated learners.  
We contribute both theoretically and empirically. On the theory side, we devise the hierarchical framework of GR2 through probabilistic graphical models and prove  the existence of a perfect Bayesian equilibrium.
 Within the GR2, we propose a practical actor-critic solver, and demonstrate its  convergent property to a stationary point in two-player games through Lyapunov analysis. On the empirical side, we validate our findings on a variety of MARL benchmarks. Precisely, we first illustrate the hierarchical thinking process  on the Keynes Beauty Contest, and then demonstrate significant improvements compared to state-of-the-art opponent modeling baselines on the  normal-form games and  the cooperative navigation benchmark. 

\end{abstract}

 \vspace{-5pt}
\section{Introduction}
\label{intro}

%
%
%
%



In people's decision making, rationality can often be compromised; it can be constrained by either the difficulty of the decision problem or the finite resources available to each individual's mind.
In behavioral game theory, instead of assuming people are perfectly rational, \emph{bounded rationality} \cite{simon1972theories} serves as the alternative modeling basis by recognizing such cognitive limitations.
One most-cited example that bounded rationality prescribes is  Keynes Beauty Contest \cite{Keynes:1936}. 
In the contest,  all players are asked to pick one number from $0$ to $100$, and the player whose guess is closest to $\nicefrac{1}{2}$ of the average number eventually becomes the winner. 
 In this game, if all the players are perfectly rational, the only choice is to guess $0$ (the only Nash equilibrium) because each of them could reason as follows: ``\emph{if all players guess randomly, the average of those guesses would be $50$ ($level$-$0$), I, therefore, should guess no more than $\nicefrac{1}{2} \times 50 = 25$ ($level$-$1$), and then if the other players think similarly as me, I should guess no more than $\nicefrac{1}{2} \times 25=13$ ($level$-$2$) ...}". 
 Such level of recursions can keep developing infinitely until all players guess the equilibrium $0$.
  This theoretical result from the perfect rationality is however inconsistent with the experimental finding in psychology \citep{coricelli2009neural} which suggests that most human players would choose between $13$ and $25$.
  In fact, it has been shown  that human beings  tend to reason only by $1$-$2$ levels of recursions in strategic games \cite{camerer2004cognitive}.
In the  Beauty Contest, players' rationality is bounded and their behaviors are  sub-optimal. As a result, it  would be unwise to guess the Nash equilibrium $0$ at all times.

In the multi-agent reinforcement learning (MARL), one common assumption is that all agents behave rationally \cite{albrecht2018autonomous} during their interactions. For example, we assume agents' behaviors will converge to Nash equilibrium \cite{yang2018mean}. 
However, in practice, it is hard to guarantee that all agents have the same level of sophistication in  their abilities of understanding and learning from each other. 
With the development of MARL methods, agents could face various types of opponents ranging from  naive independent learners~\citep{bowling2002multiagent},  joint-action learners~\citep{claus1998dynamics}, to the complicated   theory-of-mind learners  \citep{rabinowitz2018machine,shum2019theory}. 
It comes with no surprise that the effectiveness of MARL models decreases   when the opponents act irrationally \cite{shoham2003multi}.
On the other hand, it is not desirable to design agents that  can only tackle  opponents that play optimal policies. 
Justifications can be easily found in modern AI applications including self-driving cars \cite{shalev2016safe} or video game designs \cite{peng2017multiagent,hunicke2005case}.
Therefore, it becomes critical for MARL models to acknowledge different levels of bounded rationality.

In this work, we propose a novel framework -- \emph{Generalized Recursive Reasoning (GR2)} -- that recognizes agents'  bounded rationality and thus   can  model their corresponding sub-optimal behaviors.
 GR2 is inspired by cognitive hierarchy theory \cite{camerer2004cognitive}, assuming that agents could possess different levels of reasoning rationality during the interactions.
It begins with $level$-$0$ (L$0$ for short) non-strategic thinkers who do not model their opponents. 
L$1$ thinkers are more sophisticated than $level$-$0$;  they believe  the opponents are all at L$0$ and then act correspondingly.
 With the growth of $k$, L$k$ agents think in an increasing order of sophistication and then take the best response to all possible lower-level opponents.
We immerse the GR2 framework into MARL through graphical models, and derive the practical GR2 soft actor-critic algorithm. 
Theoretically, we prove the existence of Perfect Bayesian Equilibrium in the GR2 framework,  as well as  the convergence of GR2 policy gradient methods on  two-player normal-form games.
Our  proposed GR2 actor-critic methods are evaluated against multiple strong MARL baselines on  Keynes Beauty Contest, normal-form games, and cooperative navigation.
Results justify our theoretical findings and   the effectiveness  of bounded-rationality modeling.

\vspace{-5pt}
\section{Related Work}

Modeling the opponents in a recursive manner can be regarded as a special type of opponent modeling~\citep{albrecht2018autonomous}. 
Recently, studies on Theory of Mind (ToM)~\citep{goldman2012theory,rabinowitz2018machine,shum2019theory} explicitly model the agent's belief on opponents' mental states in the reinforcement learning (RL) setting. 
The I-POMDP framework focuses on building the beliefs about opponents' intentions into the planning and making agents acting optimally with respect to such  predicted intentions ~\citep{gmytrasiewicz2005framework}.
GR2 is different in that it incorporates a hierarchical structure  for  opponent modeling; it can take into account  opponents with different levels of rationality and  therefore can conduct nested reasonings about the opponents (e.g. ``I believe you believe that I believe ... "). 
 In fact, our method is most related to the probabilistic recursive reasoning (PR2) model \cite{wen2018probabilistic}. 
PR2 however only explores the $level$-$1$ structure and it does not target at modeling the bounded rationality. 
 Most importantly, PR2 does not consider whether an equilibrium exists  in such sophisticated hierarchical framework at all.
In this work, we extend the reasoning level to an arbitrary number, and theoretically prove the existence of equilibrium under the GR2  setting  as well as the convergence of the subsequent learning algorithms. 
Decision-making theorists have pointed out that the ability of thinking in a hierarchical manner is one direct consequence of the limitation in decision maker's information-processing power; they demonstrate this result by matching real-world behavioral data with the model that trades off between utility maximization against information-processing costs (i.e. an entropy term applied on the policy) \cite{genewein2015bounded}. 
Interestingly, maximum-entropy framework has  also been explored in the RL domain through inference on graphical models  \cite{levine2018reinforcement}; \emph{soft} Q-learning \cite{haarnoja2017reinforcement} and \emph{soft} actor-critic \cite{haarnoja2018soft} methods were developed.  
Recently, \emph{soft} learning has  been further   adapted into the context of MARL   \cite{wei2018multiagent,tian2019regularized}. 
In this work, we bridge the gap by embedding the solution concept of GR2  into  MARL, and derive the practical GR2 soft actor-critic algorithm. 
By recognizing  bounded rationality, we expect the GR2 MARL methods to generalize  across different types of opponents  thereby showing robustness to  their sub-optimal behaviors, 
which we believe is a critical property for modern  AI applications.  
\vspace{-5pt}
\section{Preliminaries}
\label{preliminaries}

 \emph{Stochastic Game}~\citep{shapley1953stochastic} is a natural framework to describe the $n$-agent decision-making process; it is  typically defined by the tuple $\left \langle \mathcal{S}, \mathcal{A}^1, \dots, \mathcal{A}^n, r, \dots, r^n, \mathcal{P}, \gamma \right \rangle$, where $\mathcal{S}$ represents the state space, $\mathcal{A}^i$ and $r^i(s, a^i, a^{-i})$ denote the action space and reward function of agent $i \in \{1,\dots,n\}$, $\mathcal{P}: \mathcal{S} \times \mathcal{A}^{1} \times \cdots \times \mathcal{A}^{n} \rightarrow \mathcal{P}(\mathcal{S})$ is the transition probability of the environment, and $\gamma \in (0, 1]$ a discount factor of the reward over time.  We assume agent $i$ chooses an action $a^{i} \in \mathcal{A}^{i}$ by sampling its policy $\pi_{\theta^i}^i(a^{i} | s)$ with $\theta^i$ being a tuneable parameter, and use $a^{-i} = (a^j)_{j \neq i}$ to represent actions executed by opponents. The trajectory $\tau^i = \left[(s_1, a^i_1, a^{-i}_1), \dots, (s_T, a^i_T, a^{-i}_T)\right]$  of agent $i$ is defined as a collection of state-action triples over a horizon $T$.

\subsubsection{The Concept of Optimality in MARL}
Analogous to standard reinforcement learning (RL), each agent in MARL attempts to determine an optimal policy maximizing its total expected reward. On top of RL, MARL  introduces additional complexities to the learning objective because the reward now also depends on the actions executed by opponents. 
Correspondingly, the value function of the $i_{th}$ agent in a state $s$ is  $V^i(s; \boldsymbol{\pi_\theta} )= \mathbb { E } _ { \boldsymbol{\pi_\theta}, \mathcal{P} } \left[ \sum _ { t = 1 } ^ { T } \gamma ^ { t-1 } r^i \left( s _ { t } , a _ { t } ^ { i } , a _ { t } ^ { - i } \right) \right]$ where $(a^{i}_t, a^{-i}_t) \sim \boldsymbol{\pi_{\theta}} = (\pi^{i}_{\theta^i}, \pi^{-i}_{\theta^{-i}})$ with $\boldsymbol{\pi_{\theta}}$ denoting the joint policy of all learners. 
As such, \emph{optimal} behavior in a multi-agent setting stands for acting in \emph{best response} to the opponent's policy $\pi^{-i}_{\theta^{-i}}$, which can be formally defined as the policy $\pi^i_{*}$ with $V^i(s; \pi_*^i, \pi^{-i}_{\theta^{-i}}) \ge V^i(s; \pi^i_{\theta^i}, \pi^{-i}_{\theta^{-i}})$ for all valid $\pi^i_{\theta^i}$. 
If all agents act in best response to others, the game arrives at a Nash equilibrium \cite{nash1950equilibrium}.
 Specifically, if  agents execute the policy of the form 
$
\pi^{i}(a^{i} | s)=\frac{\exp \left( Q^{i}_{\boldsymbol{\pi_\theta}}(s, a^{i}, a^{-i})\right)}{\sum_{a^{\prime}} \exp \left( Q^{i}_{\boldsymbol{\pi_\theta}}(s, a^{\prime}, a^{-i})\right)}
$
 -- a standard type of policy  adopted in RL literatures -- 
with
$ Q^{i}_{\boldsymbol{\pi_\theta}}(s, a^{i}, a^{-i}) = r^i(s, a^i, a^{-i}) + \gamma \mathbb{E}_{\mathcal{P}} [V^i(s'; \boldsymbol{\pi_\theta})]$ denoting agent $i$'s Q-function and $s^{\prime}$ being a successor state,  
they reach a Nash-Quantal equilibrium \cite{mckelvey1995quantal}.

\subsubsection{The Graphical Model of MARL}
Since GR2 is a probabilistic model, it is instructive to provide a brief review of graphical model for MARL. 
In single-agent RL, finding the optimal policy  can be equivalently  transferred into an inference problem on a graphical model \cite{levine2018reinforcement}. 
%
%
Recently, it has been shown that such equivalence also holds in the multi-agent setting \cite{tian2019regularized,wen2018probabilistic}.
To illustrate, we first introduce a  binary random variable $\mathcal{O}_{t}^{i} \in \{0,1\}$  (see Fig.~\ref{fig:chk}) that stands for the optimality of agent $i$'s policy at time $t$,  i.e., 
$p\left(\mathcal{O}_{t}^{i}=1 | \mathcal{O}_{t}^{-i}=1, \tau_t^i \right) \propto \exp \left(r^i\left(s_{t}, a_{t}^{i}, a_{t}^{-i}\right)\right)$, which suggests that given a trajectory $\tau_i^t$, the probability of being optimal is proportional to the reward.  
In the fully-cooperative setting, if all agents play optimally, then agents  receive the maximum reward that is also the Nash equilibrium; therefore, for agent $i$, it aims to maximize $p(\mathcal{O}_{1:T}^i = 1|\mathcal{O}_{1:T}^{-i}=1)$ as this is the probability of obtaining the maximum cumulative reward/best response towards Nash equilibrium.  
For simplicity, we omit the value for $\mathcal{O}_t$ hereafter. 
As we assume no knowledge of the optimal policies $\boldsymbol{\pi_{*}}$ and the model of the environment $\mathcal{P}(\mathcal{S})$, we treat them as latent variables and applied variational inference \cite{blei2006variational} to approximate such objective;  
using the variational form of $\hat{p}(\tau^{i} | \mathcal{O}^i_{1:T}, \mathcal{O}^{-i}_{1:T})=[\hat{p}(s_{1}) \prod_{t=1}^{T-1} \hat{p}(s_{t+1} | s_{t}, a_{t}^{i}, a_{t}^{-i})] \boldsymbol{\pi_{\theta}}(a_{t}^{i}, a_{t}^{-i} | s_{t}) $ leads to  
%
%
%
%
%
\vspace{-0pt}
{\small
\begin{align}
    \label{obj1}
    &\max \mathcal{J}({\boldsymbol{\pi_{\theta}}}) = \log p(\mathcal{O}_{1:T}^i = 1|\mathcal{O}_{1:T}^{-i}=1)\\  & \quad  \ge \sum_{\tau^i} \hat{p}(\tau^{i} | \mathcal{O}^i_{1:T}, \mathcal{O}^{-i}_{1:T})  \log \dfrac{p( \mathcal{O}^{i}_{1:T},\tau^{i} | \mathcal{O}^{-i}_{1:T})} {\hat{p}(\tau^{i} | \mathcal{O}^i_{1:T}, \mathcal{O}^{-i}_{1:T})  }\nonumber \\  &  \quad=\sum_{t=1}^{T} \mathbb{E}_{ \tau^{i} \sim \hat{p}(\tau^{i})} \bigg[r^{i}\big(s_{t}, a_{t}^{i}, a_{t}^{-i}\big) + \mathcal{H}\big( \boldsymbol{\pi_{\theta}} (a_{t}^{i}, a_{t}^{-i} | s_{t}) \big)\bigg]. \nonumber
\end{align} 
}

\vspace{-10pt}
To maximize $\mathcal{J}({\boldsymbol{\pi_{\theta}}}) $, a variant of policy iteration called \emph{soft learning} is applied. 
For policy evaluation, Bellman expectation equation now holds on the \emph{soft} value function 
$
V ^ i ( s )  = \mathbb{E}_{\boldsymbol{\pi_\theta} }  \left[Q^{i} (s_t, a_t^i, a_t^{-i}) -  \log ( \boldsymbol{\pi_\theta} ( a ^{i}_ { t }, a ^{-i}_ { t } | s _ { t } ) )  \right]
$, 
with the updated Bellman  operator 
$
\mathcal{T}^ { \pi } Q^{i} (s_t, a_t^i, a_t^{-i})  \triangleq r^{i}(s_t, a_t^i, a_t^{-i})    +   \tiny{\gamma \mathbb{E}_{\mathcal{P}}
\left[ \operatorname{soft} Q(s_t, a_t^{i}, a_t^{-i})\right]}.   
$
Compared to the $\max$ operation in the normal Q-learning, $\operatorname{soft}$ operator is $ \operatorname{soft} Q(s, a^i, a^{-i}) = \log \sum_{a} \sum_{a^{-i}} \exp { \left( Q ( s   , a^i, a^{-i} ) \right)}  \approx \max _ {  a^i  , a^{-i} } Q \left(  s  , a^i , a^{-i} \right)$.  
Policy improvement however becomes non-trivial because the  Q-function now guides the  improvement direction for the joint policy rather than for each single agent.
Since the exact parameter of opponent policy is usually unobservable, agent $i$  needs to approximate $\pi^{-i}_{\theta^{-i}}$.



%



\begin{figure}[t!]
     \centering
\includegraphics[width=.43\textwidth]{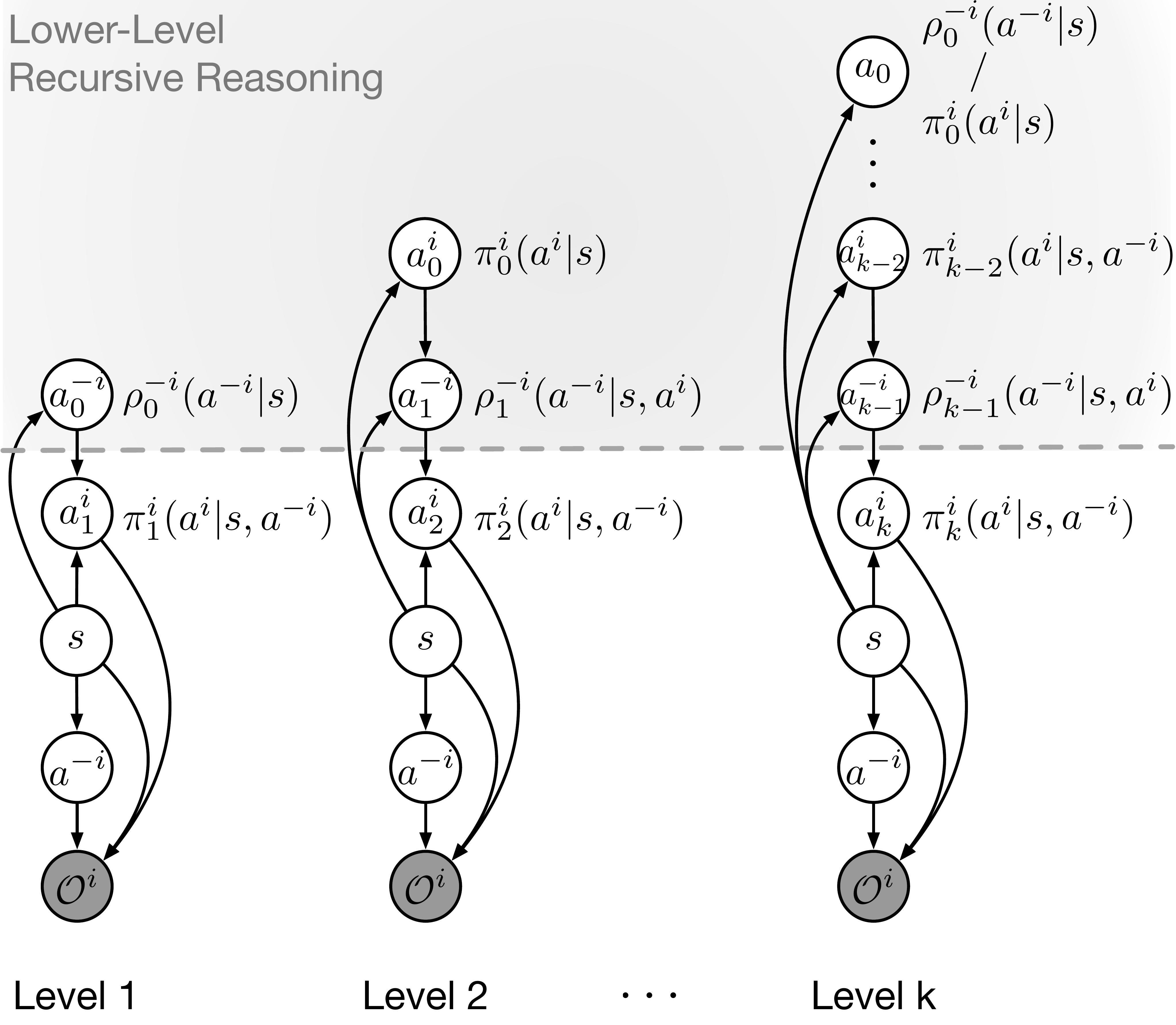}
          \vskip -7pt
     \caption{Graphical model of the $level$-$k$ reasoning model. Subfix of $a_*$  stands for the level of thinking not the timestep. The  opponent policies are approximated by $\rho^{-i}$. The omitted $level$-$0$ model  considers opponents fully randomized. Agent $i$ rolls out the recursive reasoning about opponents in its mind (grey area). In the recursion, agents with higher-level beliefs take the best response to the lower-level agents. 
}
\label{fig:chk}
      \vskip -10pt 
\end{figure}

\vspace{-5pt}
\section{Generalized Recursive Reasoning}
Recursive reasoning is essentially taking iterative best response to opponents' policies. 
$level$-$1$ thinking is ``I know you know how I know".  
We can represent such recursion by  $\boldsymbol{\pi}(a^i, a^{-i} | s) = \pi^{i}(a^{i}| s)\pi^{-i}(a^{-i}| s, a^{i})$ where  $\pi^{-i}(a^{-i}| s, a^{i})$ stands for the opponent's consideration of agent $i$'s action  $a^i \sim  \pi^{i}(a^{i}| s)$.
The unobserved opponent conditional policy $\pi^{-i}$ can be approximated via a best-fit model  $\rho^{-i}_{\phi^{-i}}$  parameterized by $\phi^{-i}$. 
By adopting $\boldsymbol{\pi_{\theta}}(a^i, a^{-i} | s) = \pi^{i}_{\theta^i}(a^{i}| s)\rho_{\phi^{-i}}^{-i}(a^{-i}| s, a^{i})$ in $\hat{p}(\tau^{i} | \mathcal{O}^i_{1:T}, \mathcal{O}^{-i}_{1:T})$ in maximizing the  Eq.~\ref{obj1}, we can solve the best-fit opponent policy by 

\vspace{-10pt}
\begin{equation}
    \label{eq:rho_pr2}
    \rho^{-i}_{\phi^{-i}}(a^{-i}| s, a^{i}) \propto  \exp \left( Q_{\boldsymbol{\pi_{\theta}}} ^ { i } (s, a^{i}, a^{-i} ) - Q_{\boldsymbol{\pi_{\theta}}} ^ { i} (s, a^{i} ) \right)   .
\end{equation}
We provide the detailed derivation  of Eq.~\ref{eq:rho_pr2} in \emph{Appendix \ref{app:eq2}}. 
Eq.~\ref{eq:rho_pr2} suggests that  agent $i$ believes his opponent will act in his interest in the cooperative games.
Based on the  opponent model in  Eq.~\ref{eq:rho_pr2}, agent $i$ can learn the best response policy by considering all possible opponent agents' actions:     
$    Q^i(s, a^i) =  \int _ { a^{-i} } \rho^{-i}_{\phi^{-i}}(a^{-i} | s, a^i) Q^i(s, a^i, a^{-i})\diff a^{-i}$, and then improve its own policy towards the direction of  
\vspace{-5pt}
{
\begin{equation}
\label{policy}
\footnotesize{
\begin{aligned}
\pi^{\prime}=\arg \min _{\pi^{\prime}} D_{\mathrm{KL}}\Bigg[\pi^{\prime}(\cdot | s_{t}) \bigg| \bigg| \dfrac{\exp ( Q^{i}_{\pi^i,\rho^{-i}}(s_{t}, a^i, a^{-i}) )}{\sum_{a'} \exp ( Q^{i}(s_{t}, a', a^{-i}) )}\Bigg].
\end{aligned}
}
\end{equation}
}%
\vspace{-5pt}




%
%


\subsubsection{Level-k Recursive Reasoning -- GR2-L}
Our goal is to extend the recursion to the $level$-$k$ ($k \ge 2$) reasoning (see Fig.~\ref{fig:chk}). 
In brief, each agent operating at level $k$ assumes that other agents are using $k-1$ level policies  and then acts in  best response.  We name this approach \textbf{GR2-L}.
%
In practice, 
the $level$-$k$ policy can  be constructed by integrating over all possible best responses from  lower-level policies    

{
\vspace{-10pt}
{\small
\begin{align}
\label{recur}
&\pi _ {k} ^ {i } ( a_k^ { i } | s ) \propto  \int_{a_ {k-1}^{-i}}  \bigg\{  \pi _ {k} ^ {i } ( a_ {k} ^ { i } | s, a_ {k-1} ^ { -i } ) \\
&\cdot \underbrace{\int_{a_{k-2} ^ {i}}  \Big[ \rho  _ {k-1}  ^ { -i } (  a _ {k-1}^ { -i } | s, a_ {k-2} ^ { i  } )   \pi _ {k-2} ^ {i} ( a_{k-2} ^ {i} | s) \Big] \diff a _ {k-2}^ { i }}_{\text{opponents of level k-1 best responds to  agent $i$ of level k-2}} \bigg\}\diff a_ {k-1}^{-i}. \notag
\end{align}
}%
}%
  When the levels of reasoning develop, we could think of  
the marginal policies $\pi_{k-2}^{i}(a^{i}|s)$   from lower levels  as the \emph{prior} and the conditional policies 
$\pi_{k}^{i}(a^{i}|s, a^{-i})$  as the \emph{posterior}.
From agent $i$'s  perspective, it believes that the opponents will take the best response to its own fictitious action $a^i_{k-2}$ that are two levels below, i.e.,  
$\rho_ {k-1} ^ {-i } (a_ {k-1}^{-i} | s) = \int \rho_ {k-1} ^ {-i } ( a_ {k-1} ^ { -i } | s, a^i_{k-2} )  \pi _ {k-2} ^ {i } ( a^i_{k-2} | s)\diff a _ {k-2}^ { i }$, 
 where $\pi_{k-2}^i$ can be further expanded by recursively using Eq.~\ref{recur} until meeting $\pi_ {0}$ that is usually assumed uniformly distributed. 
 Decisions are taken in a sequential manner. As such, $level$-$k$ model transforms the  multi-agent planning problem into a hierarchy of nested single-agent planning problems. 
%

\subsubsection{Mixture of Hierarchy Recursive Reasoning -- GR2-M}
So far, $level$-$k$ agent assumes all the opponents are at level $k-1$ during the reasoning process. We can further generalize the model to let each agent believe that the opponents can be much less sophisticated and they are distributed over all lower hierarchies ranging from $0$ to $k-1$ rather than only the level $k-1$, and then find the corresponding best response to such mixed type of agents. We name this approach \textbf{GR2-M}.

Since more computational resources are required with increasing $k$, e.g., human beings show limited amount of working memory ($1-2$ levels on average) in strategic thinkings \cite{devetag2003games}, it is reasonable to restrict the reasoning so that  fewer agents are willing to conduct the  reasoning beyond $k$ when $k$ grows  large.  We thus assume that  
\begin{assumption}
With increasing $k$, $level$-$k$ agents have an accurate guess about the relative proportion of agents who are doing lower-level thinking than them.
\label{assum}
\end{assumption}
The motivation of such assumption is to ensure that when $k$ is large,  there is no benefit for  $level$-$k$ thinkers to reason even harder to higher levels (e.g. level $k+1$), as they will almost have the same belief  about the proportion of lower level thinkers, and subsequently make similar decisions. 
In order to meet  Assumption~\ref{assum}, we choose to model the distribution of reasoning levels by  the Poisson distribution $f ( k ) = \frac{e ^ { - \lambda } \lambda ^ { k }}{ k !} $ where   $\lambda$ is the mean.
 A nice property of Poisson is that $f(k)/f(k-n)$ is inversely proportional to $k^n$ for  $1 \leq n < k$, which satisfies our need that high-level thinkers should have no incentives to think even harder.  
 We can now mix all $k$ levels' thinkings $\{\hat{\pi}^i_k\}$ into  agent's belief about its opponents at lower levels by 
\vspace{-3pt}
{\small
\begin{align}
\label{eq:chk-poisson}
&{\pi}  ^ { i, \lambda } _ {k} (a_k^i | s, a_{0:k-1}^{-i}) \notag\\
&:= \frac{e ^ { - \lambda }}{Z} \bigg( \frac{\lambda ^ { 0 }}{  0 !} {\hat{\pi}} _ {0}  ^ {i} (a_0^i | s)+ 
\cdots + \frac{ \lambda ^ { k }}{ k !} {\hat{\pi}} _ {k}  ^ {i} (a_k^i | s, a_{0:k-1}^{-i}) \bigg),
\end{align}
}%

\vspace{-10pt}
\noindent where the term  $Z=\sum_{n=1}^{k}\nicefrac{e^{-\lambda}\lambda^n}{n!}$. In practice, $\lambda$ can be set as a hyper-parameter, similar to  TD-$\lambda$ \cite{tesauro1995temporal}.

Note that GR2-L is a special case of GR2-M. As the mixture in GR2-M is Poisson distributed, we have $\frac{f(k-1)}{f(k-2)}=\frac{\lambda}{k-1}$; the model will bias towards the $k-1$ level when $\lambda \gg k$. 

\subsubsection{Theoretical Guarantee of GR2 Methods}



Recursive reasoning is essentially to let each agent take the best response to its opponents at different hierarchical levels.
A natural question to ask is does the equilibrium ever exist in GR2 settings? If so, will the learning methods ever converge? 

Here we demonstrate our \textbf{theoretical contributions} that 1) the dynamic game induced by GR2 has Perfect Bayesian Equilibrium; 2) the learning dynamics of policy gradient in GR2 is asymptotically stable in the sense of Lyapunov. 



%
%
%
%
%
\begin{theorem}
GR2 strategies extend a norm-form game into extensive-form game, and there exists a Perfect Bayesian Equilibrium (PBE) in that game.  
\label{theorempeb}
\end{theorem}

\vspace{-10pt}
\begin{proof}[Proof (of sketch)]See \emph{Appendix~\ref{peb}} for the full proof.
We can  extend the $level$-$k$  reasoning procedures at one state to an extensive-form game with perfect recall. 
We prove  the existence of PBE by showing both the requirements of \emph{sequentially rational} and \emph{consistency}  are met.
\end{proof}
\vspace{-5pt}

\begin{theorem}
In two-player normal-form games, if these exist a mixed strategy equilibrium, under mild conditions, the convergence of GR2 policy gradient to the equilibrium is asymptotic stable in the sense of Lyapunov. 
\label{main_t}
\end{theorem}

\begin{table*}[t!]
 \vskip -30pt
\caption{The Converging Equilibrium on  Keynes Beauty Contest.}
\label{table:beauty}
\vspace{-10pt}
\begin{center}
\begin{small}
\begin{sc}
\resizebox{1. \linewidth}{!}{
\begin{tabular}{lrrrrrrrrrrr}
\hline
\multicolumn{2}{l}{recursive depth }   & Level 3 & \multicolumn{1}{|c}{Level 2} & \multicolumn{3}{|c}{Level 1} & \multicolumn{4}{|c}{Level 0} \\ \hline
 Exp. Setting  &  Nash  & GR2-L3 & GR2-L2 & GR2-L1& PR2 & DDPG-ToM & MADDPG  & DDPG-OM & MASQL & DDPG \\ 
    
\hline
$p=0.7,n=2$  & $0.0$&    $\bs{0.0}$   &  $\bs{0.0}$      &  $\bs{0.0}$     &   $4.4$ &   $7.1$    &   $10.6$         & $8.7$        &  $8.3$ & $18.6$         \\ 
$p=0.7,n=10$ &$0.0$&   $\bs{0.0}$    &   $0.1$    &  $0.3$     &  $9.8$ &   $13.2$     &  $18.1$       & $12.0$ &     $8.7$  &  $30.2$     \\ 
$p=1.1,n=10$& $100.0$  &  $99.0$     &  $94.2$     &   $92.2$    &  $64.0$  &   $63.1$     &   $68.2$        & $61.7$ &  $87.5$ &   $52.2$       \\ 
\hline                             
\end{tabular}
}
\end{sc}
\end{small}
\end{center}
 \vspace{-10pt}
\end{table*}
\begin{figure*}[t!]
     \centering
     \begin{subfigure}[l]{.47\textwidth}
         \centering
         \includegraphics[width=\textwidth]{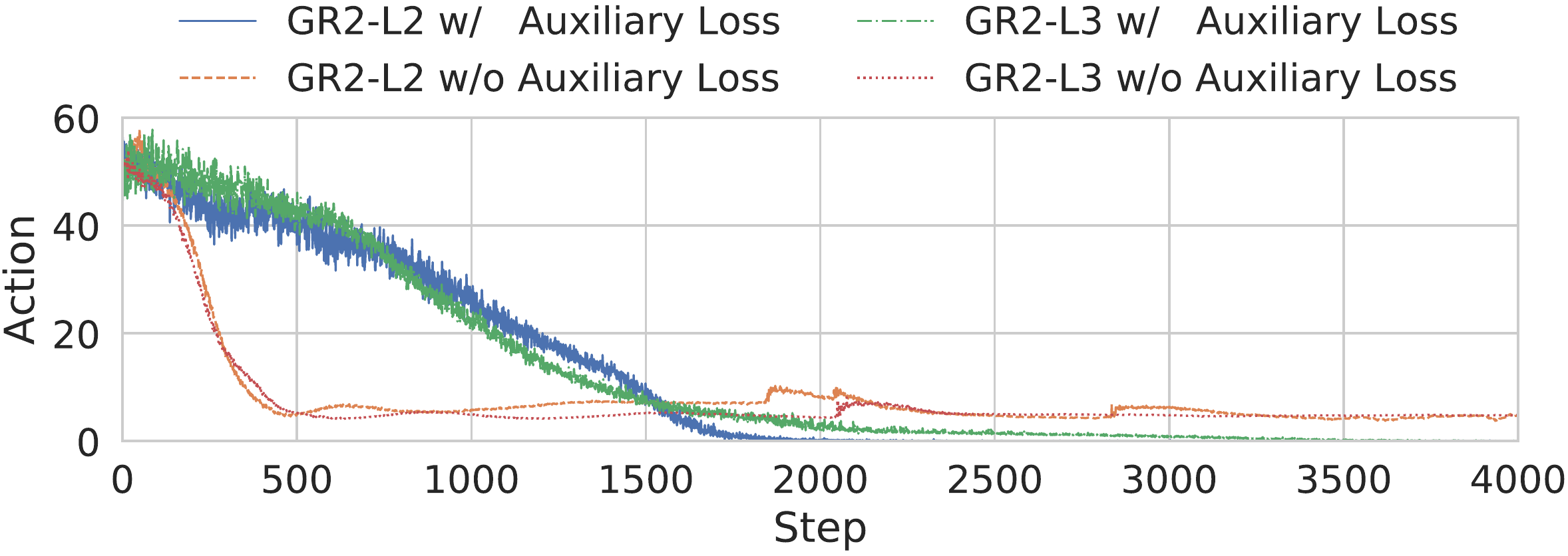}
         \caption{}
         \label{fig:pbeauty_aux_07}
     \end{subfigure}
      \begin{subfigure}[r]{.46\textwidth}
         \centering
         \includegraphics[width=\textwidth]{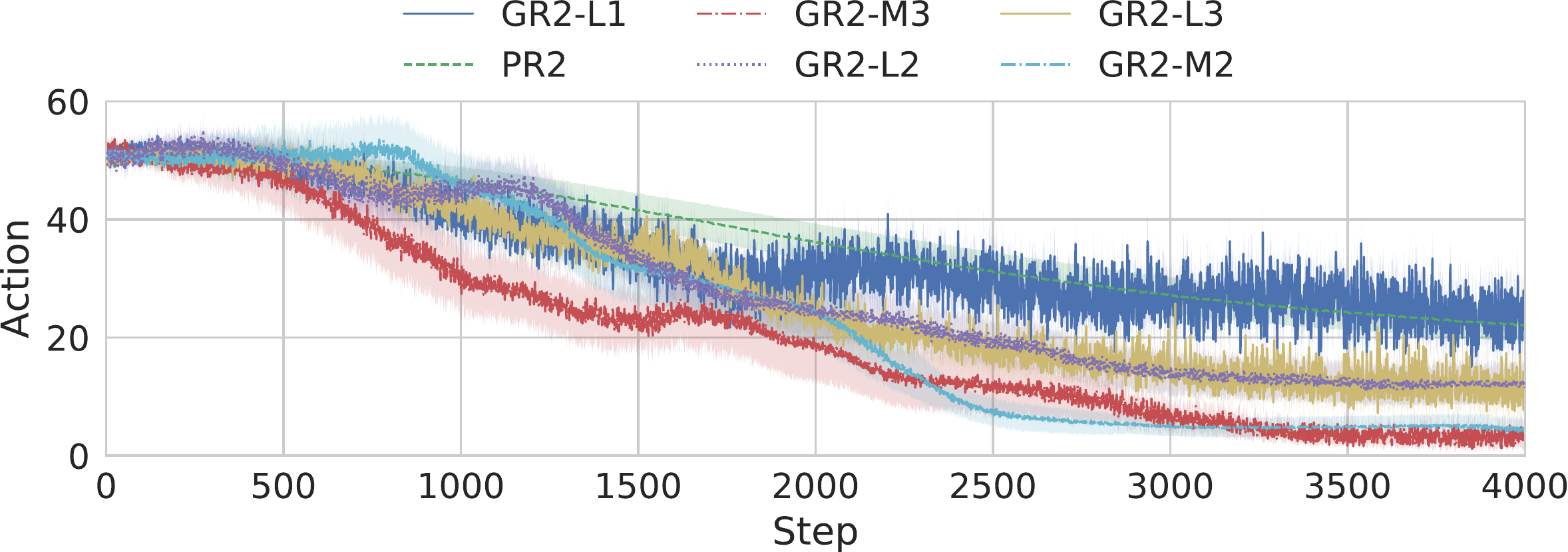}
         \caption{}
         \label{fig:pbeauty_mix}
     \end{subfigure}
      \vskip -5pt
     \caption{Beauty Contest of $p=0.7$, $n=2$. (a) Learning curves w/ or w/o the auxiliary loss of  Eq.~\ref{eq:j_level_constraint}. (b) Average learning curves of each GR2 method against the other six baselines (round-robin style).}
     \label{fig:pbeauty_roub_curve}
     \vskip -15pt
\end{figure*}
\begin{figure}
     \centering
         \includegraphics[width=.35\textwidth]{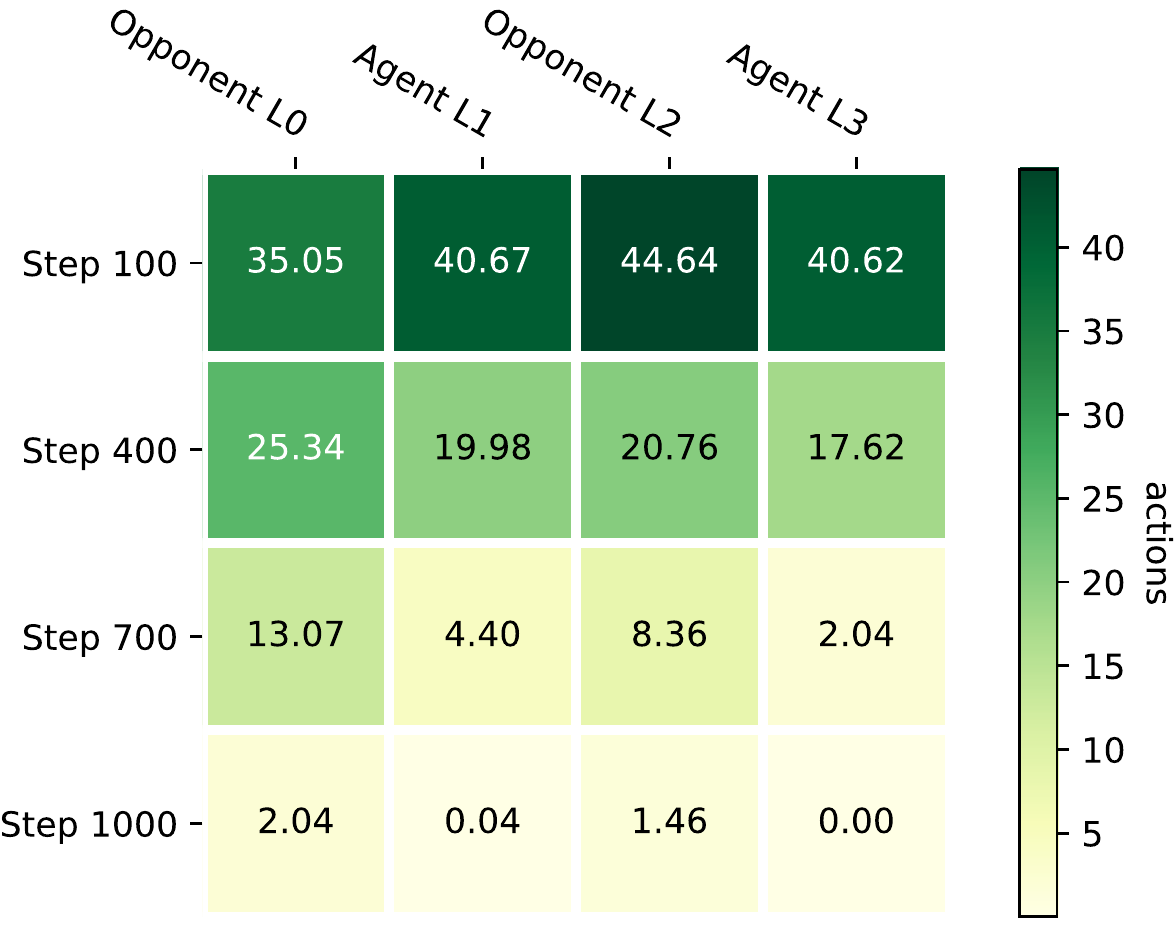}
     \vskip -5pt
     \caption{The guessing number of both agents during the training of the  GR2-L$3$ model in the Beauty Contest setting $(n=2, p=0.7)$.}
     \label{fig:k_reasoning}
     \vskip -10pt
\end{figure}

\vspace{-5pt}
\begin{proof}[Proof (of sketch)]
See \emph{Appendix \ref{lyapunov}} for the full proof.
In the two-player normal-form game,  we can treat the policy gradient update as a dynamical system. 
Through Lyapunov analysis, we first show why the convergence of $level$-$0$ method, i.e. \textbf{independent learning},  is not stable. 
 Then we show that the $level$-$k$ method's convergence is asymptotically stable as it accounts for opponents' steps before updating the policy. 
\end{proof}
\vspace{-5pt}

\begin{proposition}
In both the GR2-L \& GR2-M model, if the agents play pure strategies, once $level$-$k$ agent reaches a Nash Equilibrium, all higher-level agents will follow it too.  
\label{prop}
\end{proposition}

\vspace{-5pt}
\begin{proof}
See Appendix \ref{nash_prop} for the full proof.
\end{proof}
\vspace{-5pt}



\begin{corollary}
\label{col}
In the GR2 setting, higher-level strategies weakly dominate lower-level strategies.
\end{corollary}
\vspace{-10pt}





\vspace{-5pt}
\section{Practical Implementations}

Computing the recursive reasoning is computational expensive. Here we first present the GR2 soft actor-critic algorithm with the pseudo-code in Algo.~\ref{alg:gr2ac_sketch}, and then introduce the compromises we make to afford the implementation.

\vspace{-5pt}
\begin{algorithm}[h]
\caption{GR2 Soft Actor-Critic Algorithm}
\label{alg:gr2ac_sketch}
\begin{algorithmic}[1]
\STATE Init: $\lambda, k$ and $\psi$ (learning rates). 
\STATE Init: $\theta^i, \phi^{-i}, \omega^i$  for each agent $i$. $ \bar{\omega}^{i} \leftarrow \omega^i$,
 $\mathcal{D}^i \leftarrow \emptyset$. 
   \FOR{each episode}
        \FOR{each step $t$}
                \STATE  Agents take a step according to $\pi^{i}_{\theta^i,k}(s)$  or $\pi^{i,\lambda}_{\theta^i,k}(s)$. 
            \STATE Add experience $(s, a^i, a^{-i}, r^i, s^{\prime})$ to $\mathcal{D}^i$.
            \FOR{each agent $i$}
                \STATE Sample a batch $\{(s^{\prime}_j, a^i_j, a^{-i}_j, r^i_j, s^{\prime}_j)\}_{j=0}^{M} \sim \mathcal{D}^i$. 
                \STATE Roll out policy to level k via GR2-L/M to get $a^{i\prime}_{j}$ and record inter-level results $(a^{i\prime}_{j,k},a^{-i\prime}_{j,k-1}, \cdots)$. 
                \STATE  Sample $a^{-i\prime}_{j} \sim \rho^{-i}_{\phi^{-i}}( \cdot | s^{\prime}_j, a^{i\prime}_j)$. 
\STATE $
\omega^{i} \leftarrow \omega^{i} - \psi_{ Q^{i} } \hat { \nabla } _ { \omega^{i} } J _ { Q^{i} } ( \omega^{i} ).
$  
\STATE $
\theta^{i} \leftarrow \theta^{i} - \psi_{ \pi^{i} } \hat { \nabla } _ { \theta^{i} }\left( J _ { \pi^{i}_k } ( \theta^{i} ) + J _ { \pi^{i}_{\tilde{k}} } ( \theta^{i} )\right).
$  
\STATE $
\phi^{-i} \leftarrow \phi^{-i} - \psi_{ \rho^{-i} } \hat { \nabla } _ { \phi^{-i}  } J _ { \rho^{-i} } ( \phi^{-i}  ).
$
            \ENDFOR
                \STATE 

$
\bar{\omega} ^ { i } \leftarrow \psi_{\bar{\omega} } \omega ^ { i } + ( 1 - \psi_{\bar{\omega}}  ) \bar{\omega} ^ { i  }.
$
        \ENDFOR
   \ENDFOR
\end{algorithmic}
\end{algorithm}
  \vspace{-10pt}



\paragraph{GR2 Soft Actor-Critic.}

For policy evaluation, 
each agent  rolls out the  reasoning policies recursively to level $k$ by either Eq.~\ref{recur} or  Eq.~\ref{eq:chk-poisson},  
 the parameter $\omega^{i}$ of the joint soft $Q$-function is then updated via minimizing the soft Bellman residual 
    $J_ { Q ^{i} } ( \omega ^ {i} )
    = \mathbb { E } _ {\mathcal { D }^{i} } [ \frac { 1 } { 2 } ( Q^{i} _ { \omega ^ {i} } ( s , a ^{i} ,  a ^{-i} )  - \hat{Q}^{i}( s , a ^{i} ,  a ^{-i} ) ) ^ { 2 } ]
    $
where $\mathcal{D}^i$ is the replay buffer storing trajectories, and the target $\hat{Q}^{i}$ goes by $\hat{Q}^{i}( s , a ^{i} ,  a ^{-i} ) = r^{i}( s , a ^{i} ,  a ^{-i} ) +  \gamma  \mathbb { E }_{s^{\prime} \sim \mathcal{P}}[V^{i}(s^{\prime})]$. 
In computing $V^{i}(s^{\prime})$, since agent $i$ has no access to  the exact opponent policy $\pi^{-i}_{\theta^{-i}}$, we instead compute the soft $Q^{i}(s,a^{i})$ by marginalizing the joint $Q$-function via the estimated opponent model $\rho^{-i}_{\phi^{-i}}$ by 
$Q^{i}(s,a^{i}) = \log \int \rho^{-i}_{\phi^{-i}}(a^{-i}|s, a^{i})\exp \left( Q^{i} (s, a^{i}, a^{-i}) \right) \diff a^{-i}$; the value function of the $level$-$k$ policy $\pi^{i}_{k}(a^{i}|s)$ then comes as  
$V^{i} ( s  ) = \mathbb { E } _ { a ^{i}\sim \pi^{i}_{k} } \left[ Q^i ( s , a^{i} ) - \log \pi^{i}_{k} ( a^{i}  | s  ) \right].$ 
 Note that $\rho^{-i}_{\phi^{-i}}$ at the same time is also conducting recursive reasoning against agent $i$ in the format of Eq.~\ref{recur} or Eq.~\ref{eq:chk-poisson}.
From agent $i$'s perspective however,  the optimal opponent model $\rho^{-i}$ still follows Eq.~\ref{eq:rho_pr2} in the multi-agent soft learning setting. We can therefore update $\phi^{-i}$ by minimizing the KL,
    $
    J _ { \rho ^ {-i} } ( \phi ^ {i} ) = \mathcal{D} _ { \mathrm { KL } }\big[ \rho^{-i}_{\phi^{-i}}(a^{-i}| s, a^{i})\| \exp \left( Q ^ { i }_{\omega^i} ( s, a^{i}, a^{-i} ) - Q _{\omega^i}^ {i} (s, a^{i} )  \right) \big].
    $
We maintain two approximated $Q$-functions of $Q ^ { i }_{\omega^i} ( s, a^{i}, a^{-i} ) $ and $Q ^ {i}_{\omega^i} (s, a^{i} )  $ separately for  robust training, and the gradient of $\phi^{-i}$ is computed via SVGD \citep{liu2016stein}. 

Finally, the  policy parameter $\theta^{i}$ for agent $i$ can be learned by improving towards what the current Q-function $Q ^ {i}_{\omega^i} (s, a^{i} )  $ suggests, as shown in Eq.~\ref{policy}.
By applying the reparameterization trick $a^i=f_{\theta^i}(\epsilon; s)$ with $\epsilon \sim \mathcal{N}(\boldsymbol{0}, \boldsymbol{I})$, we have $
    J _ { \pi ^ {i}_k } ( \theta ^ {i} ) =  \mathbb { E } _ {\substack{ s  , a ^ { i }_k , \epsilon  }} [ \log \pi ^{i} _ { \theta^{i},k } \left( f_{\theta^i}(\epsilon; s) | s   \right) - Q^{i}_{\omega^i}\left(s,f_{\theta^i}(\epsilon; s)\right) ].
    $
Note that as the agent' final decision comes from the best response to all lower levels,  
we would expect the gradient of $\partial {J _ { \pi ^ {i}_k }}/{\partial \theta^i}$ propagated from all higher levels during training.

\paragraph{Approximated Best Response via Deterministic Policy.}
 As the reasoning process of GR2 methods involves iterated usages of $\pi ^{i}_k ( a ^ {i} | s,  a ^ {-i} )$ and $\rho ^{-i}_k ( a ^ {-i} | s,  a ^ {i} )$, should they be stochastic,  the cost of integrating possible actions from  lower-level agents would be unsustainable for large $k$. 
 Besides, the reasoning process is also affected by the  environment where stochastic policies could further amplify the variance. 
 Considering such computational challenges, we approximate by deterministic policies throughout the recursive rollouts, e.g., the mean of Gaussian policy.
 However, note that  the highest-level agent policy $\pi_k^i$ that interacts with the environment is still stochastic. 
 To mitigate the potential weakness of   deterministic policies, we enforce the inter-level policy improvement.
The intuition comes from the Corollary.~\ref{col} that higher-level policy should perform better than lower-level policies against the opponents. 
To maintain this property, 
we introduce an auxiliary loss $J_ { \pi^{i}_{\tilde{k}}} ( \theta ^ {i} )$ in training   $\pi^{i}_{\theta^i}$ (see Fig.~\ref{fig:kdq} in \emph{Appendix \ref{sec:impl}}),  $ \mbox{with}~ s \sim \mathcal { D }^{i}, a ^{i}_{\tilde{k}} \sim  \pi^i_{\theta ^ {i}}, a ^{-i}_{\tilde{k}}  \sim \rho^{-i}_{\phi^{-i}} ~\mbox{and}~ \tilde{k}, \geq 2.$, we have 
{\footnotesize
\begin{align}
\label{eq:j_level_constraint}
&J_ { \pi^{i}_{\tilde{k}}} ( \theta ^ {i} ) =  \mathbb { E } _ {s,  a ^{i}_{\tilde{k}} ,  a ^{-i}_{\tilde{k}} } \big[ Q^{i}(s, a^{i}_{\tilde{k}},  a^{-i}_{\tilde{k} - 1}) - Q^{i}(s, a^{i}_{\tilde{k}-2},  a^{-i}_{\tilde{k} - 1}) \big]. 
\end{align}
}%
As we later show in  Fig.~\ref{fig:pbeauty_aux_07}, such auxiliary loss plays a critical role in improving the performance.

\paragraph{Parameter Sharing across Levels.}
We further assume parameter sharing for each agent during the recursive rollouts, i.e., $\theta^k = \theta^{k+2}$ for all  $\pi_{\theta^k}^i$ and $\rho_{\theta^k}^{-i}$. 
However, note that the policies that agents take at different levels are still \textbf{different} because the inputs in computing high-level policies depend on integrating  different outputs from low-level policies as shown in Eq.~\ref{recur}. 
In addition, we have the constraint in Eq.~\ref{eq:j_level_constraint} that enforces the inter-policy improvement. 
Finally, in the GR2-M setting, we also introduce different mixing weights for each lower-level policy  in the hierarchy (see Eq.~\ref{eq:chk-poisson}).

 \vspace{-5pt}
\section{Experiments}
\label{exps}

We start the experiments by elaborating  how the GR2 model works on  Keynes Beauty Contest, and then 
move onto the normal-form games that have non-trivial equilibria where common MARL methods fail to converge. 
Finally, we test on the navigation task that requires effective opponent modeling. 

We compare the  GR2 methods with six types of  baselines including Independent Learner via DDPG \citep{lillicrap2015continuous}, PR2 \citep{wen2018probabilistic}, multi-agent soft-Q (MASQL)  \citep{wei2018multiagent}, and MADDPG \citep{lowe2017multi}.
We also include the opponent modeling \citep{he2016opponent}   by augmenting DDPG with an opponent module (DDPG-OM) that predicts the opponent behaviors in future states, and a theory-of-mind model  \cite{rabinowitz2018machine} that 
captures the dependency of agent's policy on opponents' mental states (DDPG-ToM).  
We denote $k$ as the \emph{highest} level of reasoning in  GR2-L/M, and adopt $k=\{1,2,3\}, \lambda=1.5$.  All results are reported with $6$ random seeds. We leave the  detailed hyper-parameter settings and    ablation studies in \emph{Appendix \ref{exp_detail}} due to space limit.

\begin{figure*}[t!]
 \vskip -30pt
     \centering
     \begin{subfigure}[b]{.28\textwidth}
         \centering
         \includegraphics[width=\textwidth]{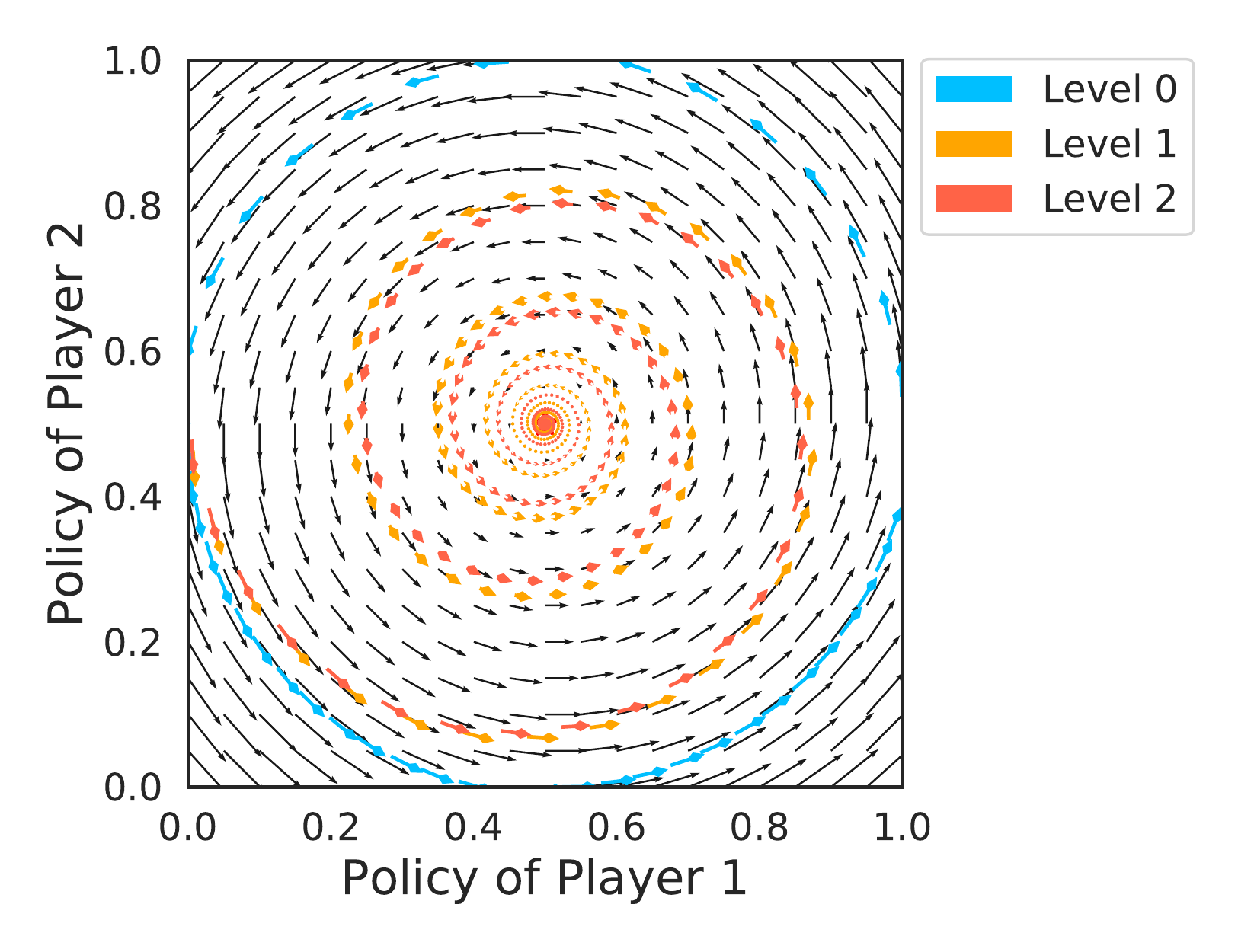}
         \caption{}
         \label{fig:l123_rotational}
     \end{subfigure}
     \begin{subfigure}[b]{.37\textwidth}
         \centering
         \includegraphics[width=\textwidth]{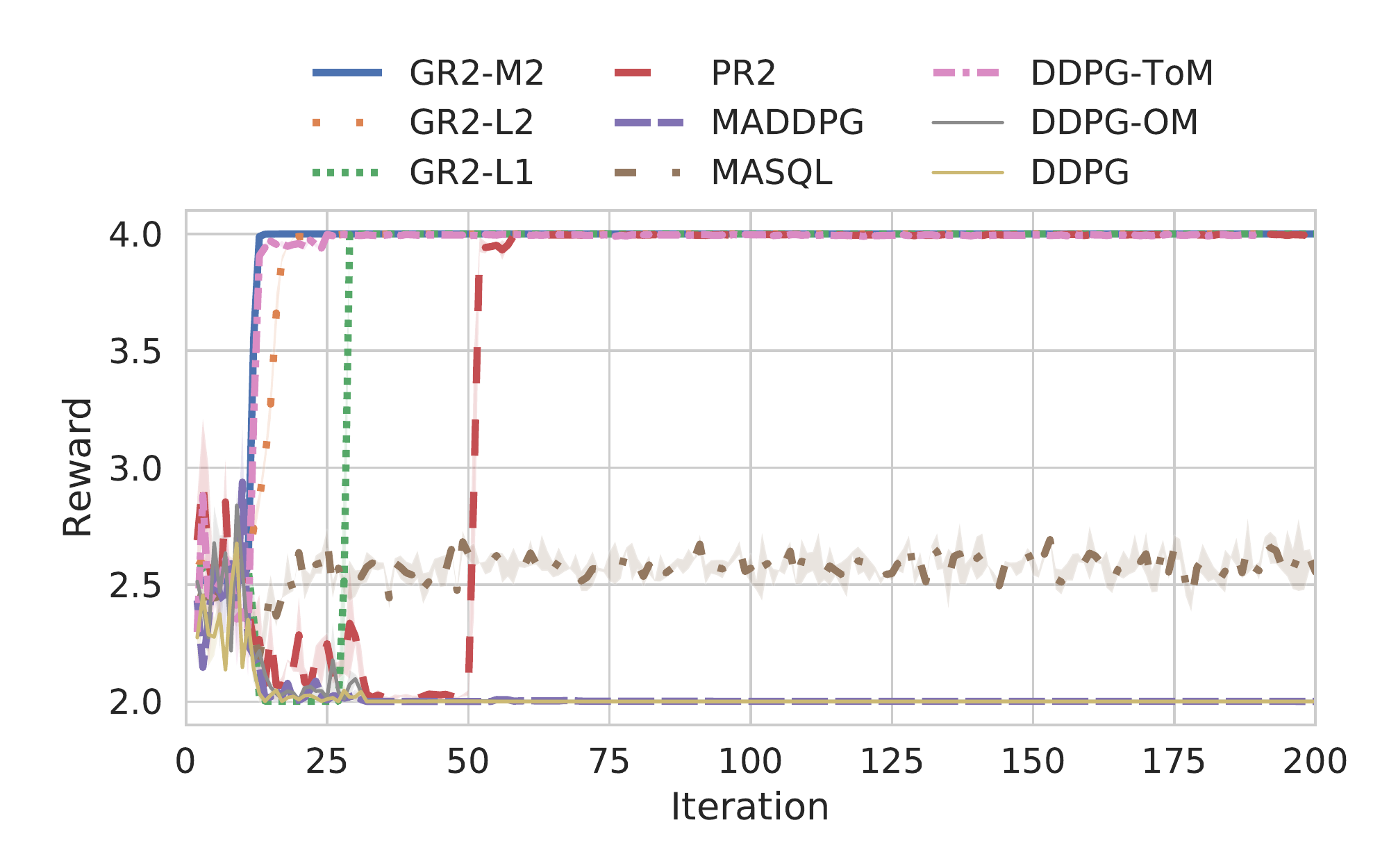}
         \caption{}
         \label{fig:ish}
     \end{subfigure}
     \begin{subfigure}[b]{.3\textwidth}
         \centering
         \includegraphics[width=\textwidth]{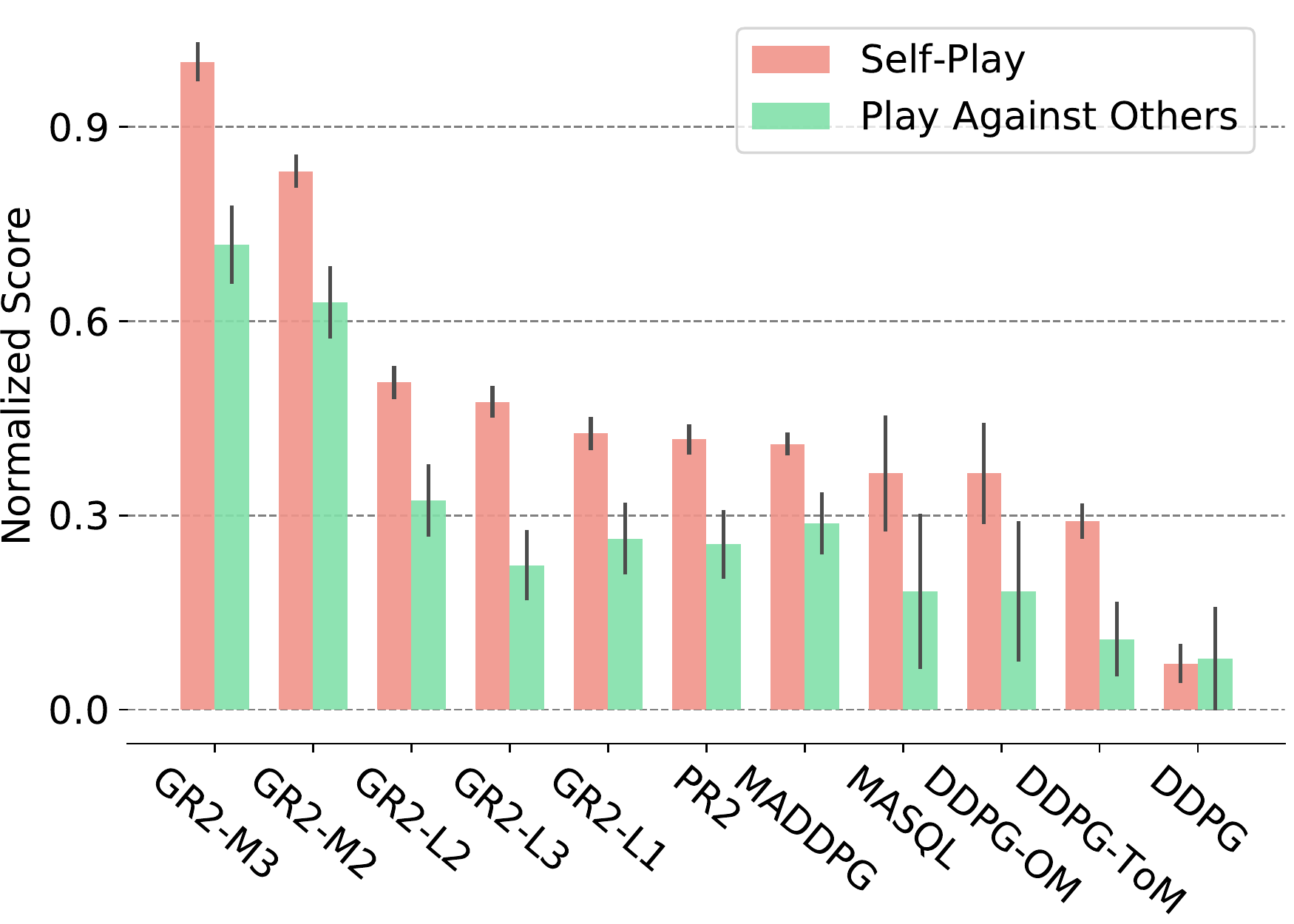}
         \caption{}
         \label{fig:particle_coop}
     \end{subfigure}
               \vskip -10pt
     \caption{(a) Learning dynamics of GR2-L on Rotational Game. (b) Average reward on Stag Hunt. (c) Performance on Coop. Navigation.}
     \label{fig:matrix}
     \vspace{-12pt}
\end{figure*}

\paragraph{Keynes Beauty Contest.}
In Keynes Beauty Contest $(n, p)$, all $n$ agents pick a number between $0$ and $100$, the winner is the agent whose guess is closest to $p$ times of the average number. 
The reward is set as the absolute difference.

In reality, higher-level thinking helps humans to get close to the Nash equilibrium of Keynes Beauty Contest (see  \emph{Introduction}). 
To validate if higher $level$-$k$ model would make multi-agent learning more effective, 
we vary different $p$ and $n$ values and present the self-play results in Table.~\ref{table:beauty}. 
We can tell that the GR2-L algorithms can effectively approach the equilibrium while the other baselines struggle to reach. The only exception is   $99.0$ in the case of $(p=1.1, n=10)$, which we believe is because of the saturated gradient from the reward.

We argue that the synergy of agents' reaching the equilibria   in this game only happens when the learning algorithm is able to make agents acknowledge different levels of rationality.  
For example, we visualize the step-wise reasoning outcomes of GR2-L$3$  in Fig.~\ref{fig:k_reasoning}. 
During training, the agent shows ability to respond to his estimation of the opponent's action by guessing a smaller number, e.g., in step $400$, $19.98<25.34$ and $17.62<20.76$. 
Even though the opponent estimation is not be accurate yet ($20.76 \neq 19.98 \times 1.1$), the agent still realizes that, with the recursive level increases, the opponent's guessing number will become smaller, in this case, $20.76 < 25.34$. Following this logic, both agents finally reach to $0$.
In addition, we  find that in $(p=0.7, n=2)$, GR2-L$1$ is  soon followed by the other higher-level GR2 models once it reaches the equilibria; this is in line with the Proposition~\ref{prop}.

To evaluate the robustness of GR2 methods outside the self-play context, we make each GR2 agent play against  all the other six baselines by a  round-robin style  and present the averaged performance in Fig.~\ref{fig:pbeauty_mix}. 
GR2-M models out-perform all the rest models by successfully guessing the right equilibrium, which is expected since the GR2-M is by design capable of  considering different types of  opponents. 

Finally, we justify the necessity of adopting the auxiliary loss of Eq.~\ref{eq:j_level_constraint} by Fig.~\ref{fig:pbeauty_aux_07}.
As we simplify the reasoning rollouts by using deterministic policies, we believe adding the auxiliary loss in the objective can effectively mitigate the potential weakness of policy expressiveness and guide the joint $Q$-function to a better direction to  improve the policy $\pi_k^i$.

\paragraph{Normal-form Games.}
We further evaluate the GR2 methods on two normal-form games: 
Rotational Game (RG) 
and  Stag Hunt (SH). 
The reward matrix of RG is $R_{\text{RG}} = \left[ \begin{array} { l l } { 0 , 3} & { 3, 2 } \\ { 1, 0 } & { 2, 1 } \end{array} \right]$, with  the only equilibria at $(0.5, 0.5)$. 
In SH, the reward matrix is $R _{\text{SH}} = \left[ \begin{array} { l l } { 4,4 } & { 1,3 } \\ { 3,1 } & { 2,2 } \end{array} \right]$. SH has   
two  equilibria  (S, S) that is Pareto optimal and (P, P) that is deficient. 



In RG, we examine the effectiveness that $level$-$k$ policies can converge to the equilibrium but $level$-$0$ methods cannot. 
We plot the gradient dynamics of RG in  Fig.~\ref{fig:l123_rotational}. 
$level$-$0$ policy, represented by independent learners,  gets trapped into the looping dynamics that never converges, while 
 the GR2-L policies can converge to the center equilibrium, with higher-level policy allowing faster speed.
These empirical findings in fact  match the theoretical results on different  learning dynamics demonstrated in the \textbf{proof of   Theorem~\ref{main_t}}.


To further evaluate the superiority of $level$-$k$ models,  we present 
Fig. \ref{fig:ish} that compares the average reward on the SH game where two equilibria exist. 
GR2 models, together with PR2 and DDPG-ToM, can  reach the Pareto optima with the maximum reward $4$, whereas other models are either fully trapped in the deficient equilibrium or mix in the middle. 
SH is a coordination game with no dominant strategy; agents choose between self-interest (P, P) and social welfare (S, S). 
Without knowing the opponent's choice, GR2 has to first anchor the belief  that the opponent may choose the social welfare to maximize its reward, and then reinforce this belief by passing it to the higher-level reasonings so that finally  the trust between agents can be built. 
The $level$-$0$ methods cannot develop such synergy because they cannot discriminate the self-interest from the social welfare as  both equilibria  can saturate the  value function. 
On the convergence speed in Fig. \ref{fig:ish}, as expected, higher-level models are faster than lower-level methods, and GR2-M  models are faster than GR2-L models.





\paragraph{Cooperative Navigation.}

We  test the GR2 methods in more complexed Particle World environments \citep{lowe2017multi} for the high-dimensional control task of  
\textit{Cooperative Navigation} with $2$ agents and $2$ landmarks. 
Agents are collectively rewarded based on the proximity of any one of the agent to the closest landmark while penalized for collisions.
The comparisons are shown in Fig.~\ref{fig:particle_coop} where we report the averaged minimax-normalized score.
We compare both the self-play performance and the averaged performance of playing with the rest $10$  baselines one on one. 
We notice that  the GR2 methods achieve critical advantages over traditional baselines in both the scenarios of self-play and playing against others; this is inline with the previous findings that GR2 agents are good at managing different levels of  opponent rationality (in this case, each opponent may want to go to a different landmark) so that collisions are  avoided at maximum. 
 In addition, we can find that all the listed models show better self-play performance than that of playing with the others; intuitively, this is because the opponent modeling is more accurate during self-plays.  


 \vspace{-5pt}
\section{Conclusion}
We have proposed a new solution concept to MARL -- generalized recursive reasoning (GR2) -- that enables agents to recognize opponents' bounded rationality and their corresponding  sub-optimal  behaviors.  
GR2 establishes a reasoning hierarchy among agents, based on which   
we derive the practical GR2 soft actor-critic algorithm.
Importantly, we prove in theory the existence of Perfect Bayesian Equilibrium  under the GR2 setting as well as the convergence of the policy gradient  methods on the two-player normal-form games. 
Series of experimental results have justified the advantages of GR2 methods over strong MARL baselines on modeling different opponents.






\clearpage

\bibliographystyle{named}
\begingroup
\bibliography{references}
\endgroup

\appendix
\onecolumn


\setcounter{secnumdepth}{1}
\renewcommand\thesection{\Alph{section}}
\section*{Appendix}

\setcounter{proposition}{0}
\setcounter{theorem}{0}

\section{Maximum Entropy Multi-Agent Reinforcement Learning}
\label{app:eq2}
We give the overall optimal distribution $p(\tau^i) = p(a^i_{1:T}, a^{j}_{1:T}, s_{1:T})$  of agent $i$ at first: 
\begin{align}
    p(a^i_{1:T}, a^{j}_{1:T}, s_{1:T}) = [ p ( s _ { 1 } ) \prod _ { t = 1 } ^ { T } p ( s _ { t + 1 } | s _ { t } , a ^{i} _ { t }, a ^{-i} _ { t } ) ] \exp ( \sum _ { t = 1 } ^ { T } r ^ {i} ( s _ { t } , a _ { t }, a ^{-i} _ { t } )) .
\end{align}
 
Analogously, we factorize empirical trajectory distribution $q(\tau^i)$ as:
\begin{equation}
\begin{aligned}
    \hat{p}(\tau^i) & =p(s_1) \prod_{t}p(s_{t^{\prime}}|s_t, a_t) \pi^i(a^{i}_t|s_t)\rho^{-i}(a^{-i}_t|s_t, a^{i}_t),
\end{aligned}
\end{equation}
where $\rho^{-i}(a^{-i}_t|s_t, a^{i}_t)$ is agent $i$'s model about the opponent's conditional policy, and $\pi^{i}(a^{i}_t|s_t)$ marginal policy of agent $i$. With fixed dynamics assumption, we can minimize the KL-divergence as follow:

\begin{equation}
\begin{aligned}
- D _ { \mathrm { KL } } ( \hat{p} ( \tau^i ) \| p ( \tau^i ) ) =\; & \mathbb{E} _ { \tau^i \sim \hat{p} ( \tau^i ) } \Big[ \log p ( s _ { 1 } ) + \sum _ { t = 1 } ^ { T } \Big( \log p ( s _ { t + 1 } | s _ { t } , a _ { t }, a ^{-i} _ { t } ) + r ^ i ( s _ { t } , a ^{i} _ { t }, a ^{-i} _ { t } ) \Big)  \\
& \qquad\qquad - \log p ( s _ { 1 } ) - \sum _ { t = 1 } ^ { T } \Big( \log p ( s _ { t + 1 } | s _ { t } , a^{i} _ { t }, a^{-i} _ { t } ) + \log \big(\pi^i(a^{i}_t|s_t)\rho^{-i}(a^{-i}_t|s_t, a^{i}_t)\big)\Big)  \Big]\\
=\;& \mathbb{E} _ { \tau^i \sim \hat{p} ( \tau^i ) } \Big[ \sum _ { t = 1 } ^ { T } r ^ i ( s _ { t } , a ^{i}_ { t }, a ^{-i}_ { t } ) - \log \big(\pi^i(a^{i}_t|s_t)\rho^{-i}(a^{-1}_t|s_t, a^{i}_t)\big)   \Big]\\    
=\;&  \sum _ { t = 1 } ^ { T } \mathbb{E} _ { ( s _ { t } , a ^{i}_ { t }, a ^{-i}_ { t }  ) \sim \hat{p} ( s _ { t } , a ^{i}_ { t }, a ^{-i}_ { t }   ) } \Big[ r ^ i ( s _ { t } , a ^{i}_ { t }, a ^{-i}_ { t }  ) - \log \big(\pi^i(a^{i}_t|s_t)\rho^{-i}(a^{-1}_t|s_t, a^{i}_t)\big) \big]  \\
=\;&  \sum _ { t = 1 } ^ { T } \mathbb{E} _ { ( s _ { t } , a ^{i}_ { t }, a ^{-i}_ { t }  ) \sim \hat{p} ( s _ { t } , a ^{i}_ { t }, a ^{-i}_ { t }  ) } \Big[ r ^ i ( s _ { t } , a ^{i}_ { t }, a ^{-i}_ { t }  ) +  \mathcal { H } \big( \rho ^ {-i} ( a ^{-i}_ { t } | s _ { t }, a ^{i}_ { t } ) \big) + \mathcal { H } \big( \pi ^ i ( a ^{i}_ { t } | s _ { t } ) \big) \Big]  ,
\end{aligned}
\end{equation}
where $\mathcal { H }$ is entropy term, and the objective is to maximize reward and polices' entropy. 

In multi-agent cooperation case, the agents work on a shared reward, which implies $\rho^{-i}(a^{-i}_t|s_t, a^{i}_t)$ would help to maximize the shared reward. It does not mean that the agent can control the others, just a reasonable assumption that the others would coordinate on the same objective. 
As before, we can find the optimal $\rho^j(a^{j}_t|s_t, a^{i}_t)$ by recursively maximizing:
\begin{equation}
\mathbb{E}_ { ( s _ { t } , a ^{i}_ { t }, a ^{-i}_ { t }  ) \sim \hat{p} ( s _ { t } , a ^{i}_ { t }, a ^{-i}_ { t }   ) }\Big[-D_{\mathrm{KL}}\Big(\rho^{-i}_t(a^{-i}_t|s_t,a^{i}_t) \Big\| \frac{\exp\big(Q^i(s_t,a^i_t,a^{-i}_t)\big)}{\exp\big(Q^i(s_t,a^{i}_t)\big)}\Big)  + Q^i(s_t,a^{i}_t)\Big],   
\end{equation}
where we define:
\begin{equation}
Q^{i} (s, a^i)  = \log \sum_{a^{-i}} \exp(Q^{i} (s, a^i, a^{-i})    ) ,
\end{equation}
which corresponds to a bellman backup with a soft maximization.  
And optimal opponent conditional policy is given as:
\begin{equation}
\rho^ {-i}(a^{-i}| s, a^{i}) \propto \exp ( Q ^ { i } (s, a^{i}, a^{-i} ) - Q ^ { i} (s, a^{i} ) ).
\end{equation}

\section{Algorithm Implementations}
\label{sec:impl}

\begin{figure*}[h!]
  \centering
  \epsfig{file=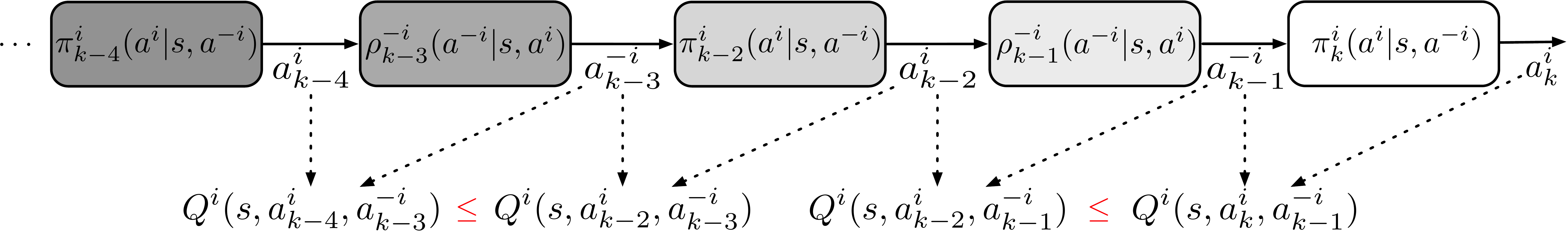, width=.75\linewidth}
  \caption{Inter-level policy improvement maintained by Eq.~\ref{eq:j_level_constraint} so that higher-level policy weakly dominates lower-level policies.
  }
\label{fig:kdq}
\end{figure*}



\section{Proof of Theorem 1}
\label{peb}

\begin{theorem}
GR2 strategies extend a norm-form game into extensive-form game, and there exists a Perfect Bayesian Equilibrium in that game.  
\end{theorem}

\begin{proof}
Consider an extensive game, which is extended from a normal form game by $level$-$k$ strategies, with perfect information and recall played by two players $(i,-i)$: $(\pi^i, \pi^{-i}, u^{i}, u^{-i}, \Lambda)$, where $\pi^{i/-i}$ and $u^{i/-i}$ are strategy pairs and payoff functions for player $i,-i$ correspondingly. $\Lambda$ denotes the  lower-level  reasoning trajectory/path so far.
An intermediate reasoning action/node in the $level$-$k$ reasoning procedure is denoted by $h_t$. The set of the intermediate reasoning actions at which player $i$ chooses to move is denoted $H^i$ (a.k.a information set). 
Let $\pi^{i/-i}_{\tilde{k}}$ denote the strategies of a $level$-$\tilde{k}$ player and $\tilde{k} \in \{0,1,2 \cdots k\}$.
A $level$-$k$ player holds a prior belief that the opponent is a $level$-$\tilde{k}$ player with probability $\lambda_{\tilde{k}}$, where $\lambda_{\tilde{k}} \in [0,1]$ and $\sum_{\tilde{k}=0}^k \lambda_{\tilde{k}} =1$.
We denote the belief that the opponent is a $level$-$\tilde{k}$ player as $p^i_{\tilde{k}}(h^t)$. 
In equilibrium, a $level$-$k$ player chooses an optimal strategy according to her belief at every decision node, which implies choice is sequentially rational as following  defined:

\begin{definition}
\label{def1}
    (Sequential Rationality). A strategy pair $\{\pi^{i}_{*},\pi^{-i}_{*}\}$ is sequentially rational with respect to the belief pair $\{p^{i},p^{-i}\}$ if for both ${i,-i}$, all strategy pairs $\{\pi^{i},\pi^{-i}\}$ and all nodes $h_t^i \in H^i$:
$$
\begin{aligned}
    \sum_{\tilde{k}=0}^k \lambda_{\tilde{k}} p^i_{\tilde{k}}(h^i_t)u^i(\pi^{i}_{*}, \pi^{-i}_{*}|h_t^i)  \geq \sum_{\tilde{k}=0}^k \lambda_{\tilde{k}} p^i_{\tilde{k}}(h^i_t)u^i(\pi^{i}, \pi^{-i}_{*}|h_t^i),
\end{aligned}
$$
\end{definition}

Based on Definition~\ref{def1}, we have the strategy $\pi^i$ is \textbf{sequentially rational} given $p^i$. It means strategy of player $i$ is optimal in the part of the game that follows given the strategy profile and her belief about the history in the information set that has occurred.

In addition, we also require the beliefs of an $level$-$k$ player are consistent. Let $p^i(h_t|\pi^{i},\pi^{-i})$ denote the probability that reasoning action $h_t$ is reached according to the strategy pair, $\{\pi^{i},\pi^{-i}\}$. Then we have the consistency definition:
\begin{definition}
\label{def2}
(Consistency). The belief pair $\{\rho^{-i}_{*},\rho^{i}_{*}\}$ is consistent with the subjective prior $\lambda_{\tilde{k}}$, and the strategy pair $\{\pi^{i},\pi^{-i}\}$ if and only if for $i,-i$ and all nodes $h_t^i \in H^i$:
$$
p_{\tilde{k},*}^{i}\left(h^{i}_{t}\right)=
\frac{\lambda_{\tilde{k}} p_{\tilde{k}}^i\left(h^{i}_{t} | \pi^{i}, \pi^{-i}\right)}
{\sum_{\hat{k}=0}^k \lambda_{\hat{k}} p_{\hat{k}}^i(h^i_t|\pi^{i}, \pi^{-i})},
$$
where there is at least one $\hat{k} \in \{0,1,2\cdots,k\}$ and $p_{\hat{k}}^i(h^i_t|\pi^{i}, \pi^{-i}) > 0$.
\end{definition}

The belief $p^i$ is \textbf{consistent} given $\pi^{i},\pi^{-i}$ indicates that for every intermediate reasoning actions reached with positive probability given the strategy profile $\pi^{i},\pi^{-i}$ , the probability assigned to each history in the reasoning path by the belief system $p^i$ is given by Bayes' rule.
In summary, sequential rationality implies each player's strategy optimal at the beginning of the game given others' strategies and beliefs~\cite{levin2019bridging}. Consistency ensures correctness of the beliefs. 

Although the game itself has perfect information, the belief structure in our $level$-$k$ thinking makes our solution concept an analogy of a Perfect Bayesian Equilibrium.
Based on above two definitions, we have the existence of Perfect Bayesian Equilibrium in $level$-$k$ thinking game.
\begin{noproposition}
For any $\lambda_{\tilde{k}}$, where $\lambda_{\tilde{k}} \in [0,1]$ and $\sum_{\tilde{k}=0}^k \lambda_{\tilde{k}} =1$,
    there is a Perfect Bayesian Equilibrium exists.
\end{noproposition}

Now, consider an extensive game of incomplete information,  $(\pi^i, \pi^{-i}, u^{i}, u^{-i}, p^{i}, p^{-i}, \lambda_k, \Lambda)$, where $\lambda_k$ denotes the possible levels/types for agents, which can be arbitrary $level$-$k$ player.
Then, according to~\citet{kreps1982sequential}, for every finite extensive form game, there exists at least one sequential equilibrium should satisfy Definition.~\ref{def1} and~\ref{def2} for sequential rationality and consistency, and the details proof as following: 

We use $E^i(\pi, p, \lambda_k,  h^i) = \sum_{\tilde{k}=0}^k \lambda_{\tilde{k}} p^i_{\tilde{k}}(h^i_t)u^i(\pi^{i}, \pi^{-i}|h_t^i)$ as expected payoff for player $i$, for every player $i$ and each reasoning path $h_t^i$.  Choose a large integer $m (m>0)$ and consider the sequence of strategy pairs and consistent belief pairs $\{\pi_m, p_m\}_m$, there exists a $(\pi_m, p_m)$:

$$
E^{i}\left(\pi_{m}, p_{m},  \lambda_k, h^i_{t^i}\right) \geq E^{i}\left((\pi_{m}^{-i},\pi^i), p_n(\pi_{m}^{-i},\pi^i),\lambda_k, h^i_{t^i}\right),
$$

for any strategy $\pi^i$ with induced probability distributions in  $\Pi_{t^i=1}^{T}=\Delta^{\frac{1}{m}}(p(h^i_{t^i}))$.

Then, consider the strategy and belief pair $\hat{\pi}, \hat{p}$ given by:

$$
(\hat{\pi}, \hat{p})=\lim _{m \rightarrow \infty}\left(\pi_{m}, p_{m}\right).
$$

Such a limit exists because $\Pi_{j=1}^{m} \Pi_{t_{j}=1}^{T_{j}} \Delta^{\frac{1}{m}}\left(p\left(h^{j}_{t_{j}}\right)\right)$ forms a compact subset  of a Euclidean space, and every sequence $\{\pi_m, p_m\}_m$ has a limit point.
We claim that for each player $i$ and each reasoning path $h^i_{t^i}$:

\begin{align}
E^{i}\left(\hat{\pi}_{m}, \hat{p}_{m},\lambda_k, h^i_{t^i}\right) \geq E^{i}\left((\hat{\pi}_{m}^{-i},\pi^i), p(\hat{\pi}_{m}^{-i},\pi^i), \lambda_k,h^i_{t^i}\right),
\label{ifnot}
\end{align}

for any strategy $\pi^i$ of player $i$. 

\textbf{If not}, then for some player $i$ and some strategy $\pi^i$ of player $i$, we have:

$$
E^{i}\left(\hat{\pi}_{m}, \hat{p}_{m},\lambda_k, h^i_{t^i}\right) < E^{i}\left((\hat{\pi}_{m}^{-i},\lambda_k,\pi^i), p(\hat{\pi}_{m}^{-i},\pi^i),\lambda_k, \lambda_k,h^i_{t^i}\right).
$$

Then, we let

$$
E^{i}\left((\hat{\pi}_{m}^{-i},\pi^i), p(\hat{\pi}_{m}^{-i},\pi^i),\lambda_k, h^i_{t^i}\right) - E^{i}\left(\hat{\pi}_{m}, \hat{p}_{m},\lambda_k, h^i_{t^i}\right) = b >0 .
$$

Now as the expected payoffs are continuous in the probability distributions at the reasoning paths and the beliefs, it follows that there is an $m_0$ sufficiently large such that for all $m\geq m_0$, 

$$
|E^{i}\left(\pi_{m}, p_{m},\lambda_k, h^i_{t^i}\right) - E^{i}\left(\hat{\pi}_{m}, \hat{p}_{m}, \lambda_k,h^i_{t^i}\right)| \leq \frac{b}{4},
$$
and
$$
E^{i}\left((\hat{\pi}_{m}^{-i},\pi^i), p_n(\hat{\pi}_{m}^{-i},\pi^i),\lambda_k, h^i_{t^i}\right) - E^{i}\left(\hat{\pi}_{m}, \hat{p}_{m}, \lambda_k,h^i_{t^i}\right) \leq \frac{b}{4}.
$$

From above equations and for all $m\geq m_0$, we have
$$
\begin{aligned}
E^{i}\left((\pi_{m}^{-i},\pi^i), p(\pi_{m}^{-i},\pi^i),\lambda_k, h^i_{t^i}\right) &\geq E^{i}\left((\hat{\pi}_{m}^{-i},\pi^i), p(\hat{\pi}_{m}^{-i},\pi^i),\lambda_k, h^i_{t^i}\right) -  \frac{b}{4}\\
&=E^{i}(\hat{\pi},\hat{p},\lambda_k,h^i_{t^i}) + \frac{3b}{4}\\
&\geq E^{i}\left((\pi_{m}^{-i},\pi^i), p(\pi_{m}^{-i},\pi^i), \lambda_k,h^i_{t^i}\right)+ \frac{b}{2}.
\end{aligned}
$$

for a given sequential game, there is a $T>0$ such that
$$
\bigg|E^{i}\left[(\pi_{\xi}^{-i},\pi^i), p_n(\pi_{\xi}^{-i},\pi^i), \lambda_k,h^i_{t^i}\right] - E^{i}\left(\hat{\pi}_{\xi}, \hat{p}_{\xi},\lambda_k, h^i_{t^i}\right) \bigg| < \frac{T}{\xi},
$$

where $\pi^{i}=\lim_{\xi \rightarrow \infty} \pi^{i}_{\xi}$ of a sequence $\{\pi^i_{\xi}\}_{\xi}$ of $\frac{1}{\xi}$ bounded  strategies of player $i$.
For the sequence $\{\pi_m,p_m\}$ we now choose an $m_1$  sufficiently large such that $\frac{T}{m}<\frac{b}{4}$. Therefore, for any strategy $\pi^i$ of player $i$, we have

$$
\begin{aligned}
E^{i}\left((\pi_{m}^{-i},\pi^i), p_n(\pi_{m}^{-i},\pi^i), \lambda_k,h^i_{t^i}\right) 
&\geq E^{i}\left((\pi_{m}^{-i},\pi^i), p(\pi_{m}^{-i},\pi^i), \lambda_k,h^i_{t^i}\right) -  \frac{T}{m}\\
&=E^{i}(\pi_m,p_m,\lambda_k,h^i_{t^i}) + \frac{b}{4}.
\end{aligned}
$$

But this result contradicts the previous claim in Eq.~\ref{ifnot}, which indicates the claim must hold.
In other words, Perfect Bayesian Equilibrium must exist.

\end{proof}

\begin{remark}
When $\lambda_k=1$, it is the special case where the policy is  $level$-$k$ strategy, and it coincides with Perfect Bayesian Equilibrium.
\end{remark}

\section{Proof of Theorem 2}
\label{lyapunov}

\begin{theorem}
In two-player two-action games, if these exist a mixed strategy equilibrium, under mild conditions, the learning dynamics of GR2 methods to the equilibrium is asymptotic stable in the sense of Lyapunov.  

\end{theorem}

\begin{proof}
We start by defining the matrix game that a mixed-strategy equilibrium exists, and they we show that on such game $level$-$0$ independent learner through iterated gradient ascent will not converge, and finally derive why the $level$-$k$ methods would converge in this case. Our tool is Lyapunov function and its stability analysis. 

Lyapunov function is used to verify the stability of a dynamical system in control theory, here we apply it in convergence proof for $level$-$k$ methods. It is defined as following:
\begin{definition}
(Lyapunov Function.) 
Give a function $F(x,y)$ be continuously differentiable in a neighborhood $\sigma$ of the origin. 
The function $F(x,y)$ is called the Lyapunov function for an autonomous system if that satisfies the following properties: 
\begin{enumerate}
  \item (nonnegative) $F(x,y)>0$ for all $(x,y) \in \sigma \string\ {(0,0)}$; 
  \item (zero at fixed-point) $F(0,0)=0$; 
  \item (decreasing) $\frac{\mathrm{d} F}{\mathrm{d} t} \leq 0$  for all $(x,y)\in \sigma$.
\end{enumerate}
\end{definition}

\begin{definition}
(Lyapunov Asymptotic Stability.) For an autonomous system, if there is a Lyapunov function $F(x,y)$ with a negative definite derivative $\frac{\mathrm{d} F}{\mathrm{d} t} < 0$ (strictly negative,  negative definite LaSalle's invariance principle) for all $(x,y) \in \sigma \string\ {(0,0)}$, then the equilibrium point $(x,y)=(0,0)$ of the system is asymptotically stable~\cite{marquez2003nonlinear}.    
\end{definition}

\noindent\textbf{Single State Game}

Given a two-player, two-action matrix game, which is a single-state stage game, we have the payoff matrices for row player and column player as follows:

$$
\label{eq:payoff_mp}
\mathbf{R}_r = \left[\begin{matrix}
r_{11} & r_{12} \\
r_{21} & r_{22}
\end{matrix}\right]
\quad
\text{and} 
\quad 
\mathbf{R}_c = \left[\begin{matrix}
c_{11} & c_{12} \\
c_{21} & c_{22}
\end{matrix}\right].
$$

Each player selects an action from the action space $\{1,2\}$ which determines the payoffs to the players. If the row player chooses action $i$ and the player 2 chooses action $j$, then the row player and column player receive the rewards $r_{ij}$ and $c_{ij}$ respectively. 
We use $\alpha\in [0,1] $ to represent the strategy for row player, where $\alpha$ corresponds to the probability of player 1 selecting the first action (action $1$), and $1-\alpha$ is the probability of choosing the second action (action $2$). Similarly, we use $\beta$ to be the strategy for column player. With a joint
strategy $(\alpha, \beta)$, the expected payoffs of players are:
%

$$
\begin{aligned} V_r(\alpha, \beta) 
=\alpha \beta r_{11}+\alpha(1-\beta) r_{12}+(1-\alpha) \beta r_{21}+(1-\alpha)(1-\beta) r_{22} ,
\end{aligned}
$$
$$
\begin{aligned} V_c(\alpha, \beta) 
=\alpha \beta c_{11}+\alpha(1-\beta) c_{12}+(1-\alpha) \beta c_{21}+(1-\alpha)(1-\beta) c_{22} .
\end{aligned}
$$

One crucial aspect to the learning dynamics analysis are the points of zero-gradient in the constrained dynamics, which they show to correspond to the equilibria which is called the center and denoted $(\alpha^*, \beta^*)$. This point can be found mathematically 
$
\left(\alpha^{*}, \beta^{*}\right)=\left(\dfrac{-b_c}{u_c}, \dfrac{-b_r}{u_r}\right)
$, where
$
u_r =r_{11}-r_{12}-r_{21}+r_{22}$, $b_r = r_{12}-r_{22}$, $ u_c =c_{11}-c_{12}-c_{21}+c_{22}$, and $b_c = c_{21}-c_{22} $.

Here we are more interested in the case that there exists a mixed strategy equilibrium, i.e., 
the location of the equilibrium point $(\alpha^*, \beta^*)$ is in the interior of the unit square, equivalently, it means $u_ru_c<0$. 
In other cases where the Nash strategy on the boundary of the unit square~\cite{marquez2003nonlinear,bowling2001convergence}, we are not going to discuss in this proof. 


\vspace{10pt}
\noindent\textbf{Learning in $level$-$0$ Gradient Ascent}

One common $level$-$0$ policy is Infinitesimal Gradient Ascent (IGA), which assumes independent learners and is a $level$-$0$ method, a player increases its expected payoff by moving its strategy in the direction of the current gradient with fixed step size. 
The gradient is then computed as the partial derivative of the agent's expected payoff with respect to its strategy, we then have the policies dynamic partial differential equations:

$$
\begin{aligned} 
\frac{\partial V_{r}(\alpha, \beta)}{\partial \alpha}
=u_r \beta +b_r, \quad
\frac{\partial V_{c}(\alpha, \beta)}{\partial \beta}
=u_c \alpha +b_c.
\end{aligned}
$$

 In the gradient ascent algorithm, a player will adjust its strategy after each iteration so as to increase its expected payoffs. This means the player will move their strategy in the direction of the current gradient with some step size. Then we can have dynamics are defined by the differential equation at time $t$:

$$
\left[ \begin{array}{c}{\nicefrac{\partial \alpha}{\partial t}} \\ {\nicefrac{\partial \beta}{\partial t}}\end{array}\right]=\underbrace{\left[ \begin{array}{cc}{0} & {u_r} \\ {u_c} & {0}\end{array}\right]}_{U}\left[\begin{array}{l}{\alpha} \\ {\beta}\end{array}\right]+\left[ \begin{array}{c}{b_r} \\ {b_c}\end{array}\right].
$$

By defining multiplicative matrix term $U$ above with off-diagonal values $u_r$ and $u_c$, we can classify the dynamics of the system based on properties of $U$. As we mentioned, we are interested in the case that the game has just one mixed center strategy equilibrium point (not saddle point) that in the interior of the unit square, which means $U$ has purely imaginary eigenvalues and $u_r u_c<0$~\cite{zhang2010multi}. 

 Consider the quadratic Lyapunov function which is continuously differentiable and $F(0,0)=0$ :
 
 $$
 F(x,y) = \frac{1}{2} (u_c x^2 - u_r y^2),
 $$
 
where we suppose $u_c > 0, u_r < 0$ (we can have identity case when $u_c < 0, u_r > 0$ by switching the sign of the function).
Its derivatives along the trajectories by setting $x = \alpha - \alpha^*$  and $y = \beta - \beta^*$ to move the the equilibrium point to origin can be calculated as:

%
%
%
%
%
%
%
%
%
$$
\begin{aligned}
\frac{\mathrm{d} F}{\mathrm{d} t}&=\frac{\partial F}{\partial x} \frac{\mathrm{d} x}{\mathrm{d} t}+\frac{\partial F}{\partial y} \frac{\mathrm{d} y}{\mathrm{d} t}
= xy(u_r u_c - u_r u_c)=0,
\end{aligned}
$$

where the derivative of the Lyapunov function is identically zero. Hence, the condition of asymptotic stability is not satisfied~\citep{marquez2003nonlinear,taylor2018lyapunov} and the IGA $level$-$0$ dynamics is unstable. There are some IGA based methods (WoLF-IGA, WPL etc.~\citep{bowling2002multiagent,abdallah2008multiagent}) with varying learning step, which change the $U$ to $\left[ \begin{array}{cc}{0} & {l_r(t)u_r} \\ {l_c(t)u_c} & {0}\end{array}\right]$. The time dependent learning steps $l_r(t)$ and $l_c(t)$ are chose to force the dynamics would converge. Note that diagonal elements in $U$ are still zero, which means
a player's personal influences to the system dynamics are not reflected on its policy adjustment.

\vspace{10pt}
\noindent\textbf{Learning in $level$-$k$ Gradient Ascent}

Consider a $level$-$1$ gradient ascent, where agent learns in term of $\pi_r(\alpha)\pi_c^1(\beta|\alpha)$, the gradient is computed as the partial derivative of the agent's expected payoff after considering the opponent will have $level$-$1$ prediction to its current strategy. We then have the $level$-$1$  policies dynamic partial differential equations:

$$
\begin{aligned} 
\frac{\partial V_{r}(\alpha, \beta_1)}{\partial \alpha}
=u_r (\beta +  \zeta \partial_{\beta} V_{c}(\alpha, \beta)))  +b_r, \quad
\frac{\partial V_{c}(\alpha_1, \beta)}{\partial \beta}
=u_c (\alpha+ \zeta \partial_{\alpha} V_{r}(\alpha, \beta)) +b_c,
\end{aligned}
$$

where $\zeta$ is short-term prediction of the opponent's strategy.
Its corresponding $level$-$1$ dynamic partial differential equations:

$$
\left[ \begin{array}{c}{\nicefrac{\partial \alpha}{\partial t}} \\ {\nicefrac{\partial \beta}{\partial t}}\end{array}\right]=\underbrace{\left[ \begin{array}{cc}{\zeta u_r u_c} & {u_r} \\ {u_{c}} & {\zeta u_r u_c}\end{array}\right]}_{U} \left[ \begin{array}{l}{\alpha} \\ {\beta}\end{array}\right]+\left[ \begin{array}{c}{\zeta u_r b_c + b_r} \\ {\zeta u_c b_r + b_c}\end{array}\right].
$$

Apply the same quadratic Lyapunov function: $ F(x,y) = 1/2 (u_c x^2 - u_r y^2)$, where $u_c > 0, u_r < 0$,
and its derivatives along the trajectories by setting $x = \alpha - \alpha^*$ and $y = \beta - \beta^*$ to move the coordinates of equilibrium point to origin:

%
%
%
%

$$
\begin{aligned}
\frac{\mathrm{d} F}{\mathrm{d} t}
= \zeta u_ru_c(u_cx^2-u_ry^2) + xy(u_ru_c-u_ru_c) = \zeta u_ru_c(u_cx^2-u_ry^2),
\end{aligned}
$$

where the conditions of asymptotic stability is satisfied due to $u_ru_c<0, u_c>0 $ and $u_r<0$, and it indicates the derivative $\frac{\diff F}{\diff t} < 0$. In addition, unlike the $level$-$0$'s case, we can find that the diagonal of $U$ in this case is non-zero, it measures the mutual influences between players after $level$-$1$ looks ahead and helps the player to update it's policy to a better position.

This conclusion can be easily extended and proved in $level$-$k$ gradient ascent policy ($k>1$) .
In $level$-$k$ gradient ascent policy, we can have the derivatives of same quadratic Lyapunov function in $level$-$2$ dynamics:

$$
\begin{aligned}
\frac{\mathrm{d} F}{\mathrm{d} t}
= \zeta u_ru_c(u_cx^2-u_ry^2) + xy\left(1 + \zeta^2u_ru_c\right)(u_ru_c-u_ru_c)=\zeta u_ru_c(u_cx^2-u_ry^2),
\end{aligned}
$$

and $level$-$3$ dynamics:

$$
\begin{aligned}
\frac{\mathrm{d} F}{\mathrm{d} t}
=\zeta u_ru_c(2 + \zeta^2u_ru_c)(u_cx^2-u_ry^2).
\end{aligned}
$$

Repeat the above procedures, we can easily write the general derivatives of quadratic Lyapunov function in $level$-$k$ dynamics:

$$
\begin{aligned}
\frac{\mathrm{d} F}{\mathrm{d} t}
&= \zeta u_ru_c(k-1 + \cdots +\zeta^{k-1}(u_ru_c)^{k-2})(u_cx^2-u_ry^2),
\end{aligned}
$$
where $k \geq 3$. These $level$-$k$ policies still owns the asymptotic stability property when $\zeta^2$ is sufficiently small (which is trivial to meet in practice) to satisfy $k -1 + \cdots + \zeta^{k-1}(u_ru_c)^{k-2} > 0$, which meets the asymptotic stability conditions, therefore coverages.
\end{proof}


%

%


\section{Proof of Proposition 1}
\label{nash_prop}
\begin{proposition}
In both the GR2-L \& GR2-M model, if the agents play pure strategies, once $level$-$k$ agent reaches a Nash Equilibrium, all higher-level agents will follow it too.  
\end{proposition}
\begin{proof}
Consider the following two cases GR2-L and GR2-M.

\textbf{GR2-L.}
Since agents are assumed to play pure strategies, if a $level$-$k$ agent reaches the equilibrium, $\pi^{i}_{k,*}$, in the GR2-L model, then all the higher-level agents will play that equilibrium strategy too, i.e. $\pi^{-i}_{k+1,*} = \pi^{i}_{k,*} $. The reason is  because high-order thinkers will conduct at least the same amount of computations as the lower-order thinkers, and $level$-$k$  model only needs to best respond to $level$-$(k-1)$. On the other hand, as it is  showed by the Eq. 3 in the main paper, higher-level recursive model contains the lower-level models by incorporating it into the inner loop of the integration.

\textbf{GR2-M.}
In the GR2-M model, if the $level$-$k$ step agent play the equilibrium strategy $\pi^{i}_{k,*}$, it means the agent finds the best response to a mixture type of agents that are among $level$-$0$ to $level$-$(k-1)$. Such strategy $\pi^{i}_{k,*}$ is at least weakly dominant over other pure strategies. For $level$-$(k+1)$ agent, it will face a mixture type of $level$-$0$ to $level$-$(k-1)$, plus $level$-$k$.

For mixture of $level$-$0$ to $level$-$(k-1)$, the strategy $\pi^{i}_{k,*}$ is already the best response by definition. For $level$-$k$, $\pi^{i,*}_{k}$ is still the best response due to the conclusion in the above GR2-L.  Considering the linearity of the expected reward for GR2-M:

$$
\small{
\begin{aligned}
    \mathbb { E }[\lambda_0V^i(s; \pi_{0,*}^i, \pi^{-i}) + \cdots + \lambda_k V^i(s; \pi_{k,*}^i, \pi^{-i})] = \lambda_0\mathbb { E }[V^i(s; \pi_{0,*}^i, \pi^{-i})] + \cdots +\lambda_k \mathbb { E }[V^i(s; \pi_{k,*}^i, \pi^{-i})],
\end{aligned}
}
$$

where $\lambda_k$ is $level$-$k$ policy's proportion.
Therefore, $\pi^{i,*}_{k}$ is the best response to the mixture of $level$-$0$ to $level$-$k$ agent, i.e. the best response for $level$-$(k+1)$ agent. Given that $\pi^{i,*}_{k}$ is the best response to both $level$-$k$ and all lower levels from $0$ to $(k-1)$, it is therefore the best response of the $level$-$(k+1)$ thinker.

Combining the above two results, therefore, in GR2, once a $level$-$k$ agent reaches a pure Nash strategy, all higher-level agents will play it too. 

\end{proof}
%
%
%
%

\section{Detailed Settings for Experiments}
\label{exp_detail}

\subsection{The Recursive Level}
We regard DDPG, DDPG-OM, MASQL, MADDPG as $level$-$0$ reasoning models because from the policy level, they do not explicitly model the 
impact of one agent's action on the other agents or consider the reactions from the other agents. Even though the value function of the joint policy is learned in MASQL and MADDPG,  but they conduct a \emph{non-correlated factorization} \cite{wen2018probabilistic} when it comes to each individual agent's policy.  
PR2 and DDPG-ToM are in fact the $level$-$1$ reasoning model, but  note that the $level$-$1$ model in GR2 stands for $\pi _ {1} ^ {i } ( a ^ { i } | s ) = \int_{a^{-i}} \pi _ {1} ^ {i } ( a ^ { i } | s, a ^ { -i } )  \rho  _ {0}  ^ { -i } (  a ^ { -i } | s ) \diff a^{-i}$, while the $level$-$1$ model  in PR2 starts from the opponent's angel, that is $\rho _ {1} ^ {-i } ( a ^ { -i } | s ) =  \int_{a^{i}}\rho _ {1} ^ {i } ( a ^ { -i } | s, a ^ { i } ) \pi  _ {0}  ^ { i } (  a ^ { i } | s )\diff a^{i}$.

\subsection{Hyperparameter Settings}
In all the experiments, we have the following parameters.
The Q-values are updated using Adam with learning rate $10^{-4}$.
The DDPG policy and soft Q-learning sampling network use Adam with a learning rate of $10^{-4}$.
The methods use a replay pool of size $100k$. Training does not start until the replay pool has at least $1k$ samples. The batch size 64 is used.
All the policies and $Q$-functions are modeled by the MLP with $2$ hidden layers followed by ReLU activation. In matrix games and Keynes Beauty Contest, each layer has $10$ units and $100$ units are set in cooperative navigation's layers.
In the actor-critic setting, we set the exploration noise to $0.1$ in the first $1k$ steps. The annealing parameter in soft algorithms is decayed in linear scheme with training step grows to balance the exploration.
Deterministic policies additional OU Noise to improve exploration with parameters $\theta=0.15$ and $\sigma=0.3$. 
We update the target parameters softly by setting target smoothing coefficient to $0.001$. 
We train with 6 random seeds for all environments. 
In  Keynes Beauty Contest, we train all the methods for $400$ iterations with $10$ steps per iteration. 
In the matrix games, we train the agents for $200$ iterations with $25$ steps per iteration. 
For the cooperative navigation, all the models are trained up to $300k$ steps with maximum $25$ episode length.

\subsection{Ablation Study}
The results in the experiment section of the main paper  suggest that GR2 algorithms can outperform other multi-agent RL methods various tasks. 
In this section, we examine how sensitive GR2 methods is to some of the most important hyper-parameters, including the $level$-$k$ and the choice of the poisson mean $\lambda$ in GR2-M methods, as well as the influences of incentive intensity in the games.

\paragraph{Choice of $k$ in  $level$-$k$ Models.}

\begin{figure}[ht!]
     \centering
     \begin{subfigure}[b]{.4\textwidth}
         \centering
         \includegraphics[width=\textwidth]{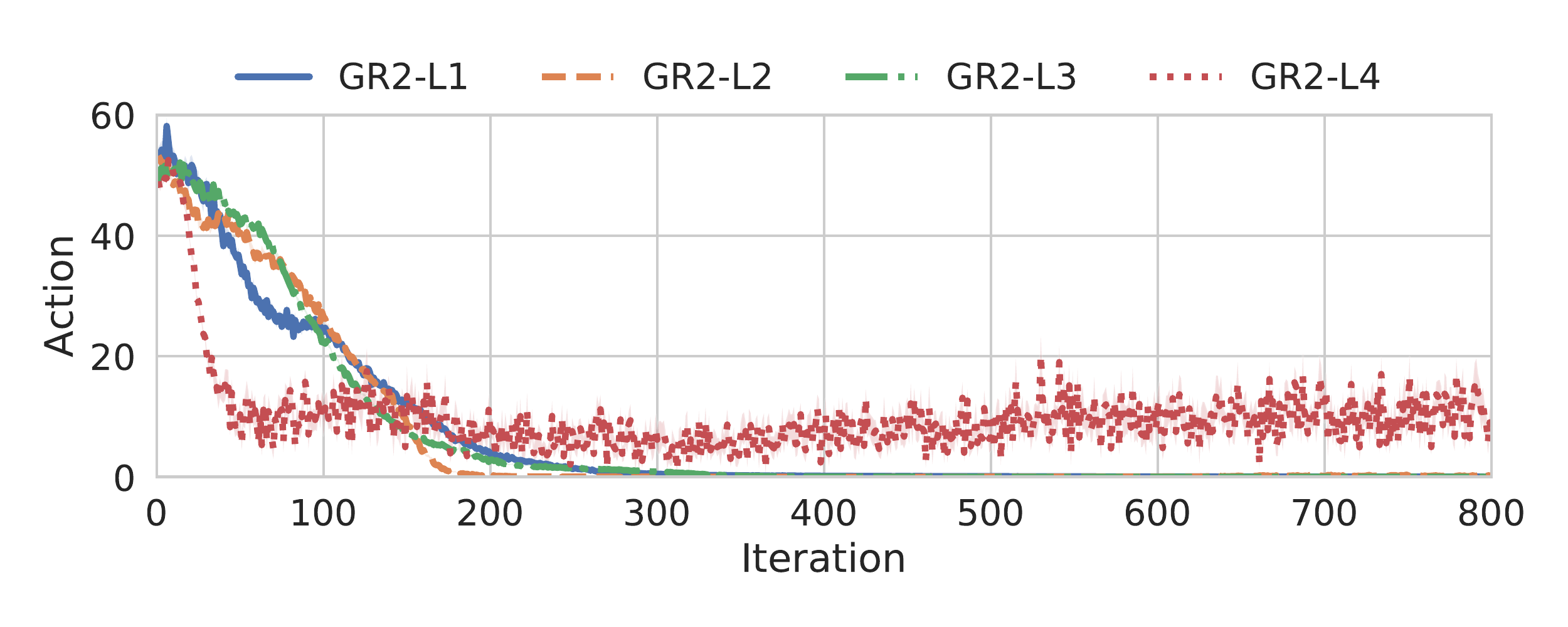}
         \caption{$p=0.7,n=10$}
         \label{abla_fig:ish}
     \end{subfigure}
     \begin{subfigure}[b]{.4\textwidth}
         \centering
         \includegraphics[width=\textwidth]{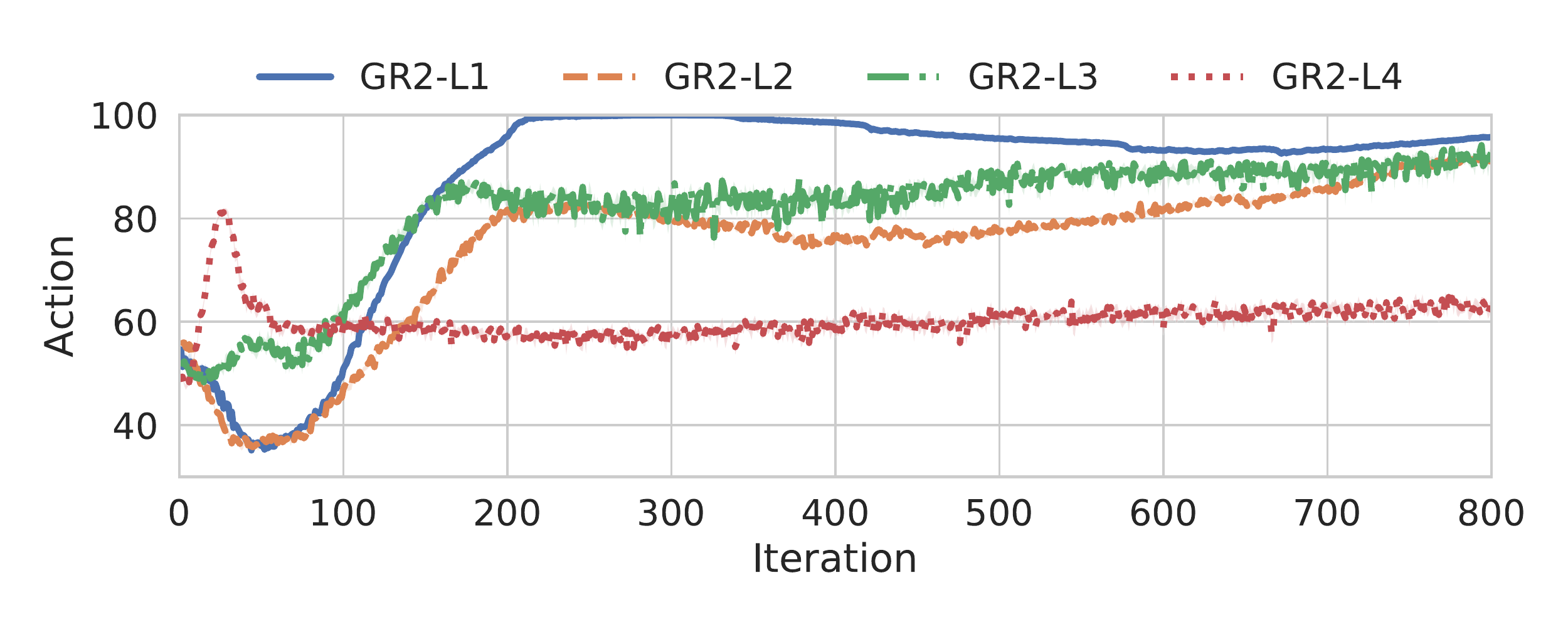}
         \caption{$p=1.1,n=10$}
         \label{fig:ipd}
     \end{subfigure}
     \caption{Learning curves on  Keynes Beauty Contest game with GR2-L policies from $level$-$1$ to $level$-$4$.}
     \label{abla_fig:pbeauty_k}
\end{figure}

First, we investigate the choice of $level$-$k$ by testing the GR2-L models with various $k$ on  Keynes Beauty Contest. 
According to the Fig. \ref{abla_fig:pbeauty_k}, in both setting, the GR3-L with level form $1-3$ can converge to the equilibrium, but the GR3-L4 cannot. 
The learning processes show that the GR3-L4 models have high variance during the learning. This phenomenon has two reasons: with k increases, the reasoning path would have higher variance; and in GR2-L4 policy, it uses the approximated opponent conditional policy $\rho^{-i}(a^{-i}|s, a^{i})$ twice (only once in GR2-L2/3), which would further amplify the variance. 
\\
\\
\\
\paragraph{Choice of $\lambda$ of Poisson Distribution in GR2-M.}

\begin{figure}[H]
     \centering
     \begin{subfigure}[b]{.4\textwidth}
     \centering
         \includegraphics[width=1.\textwidth]{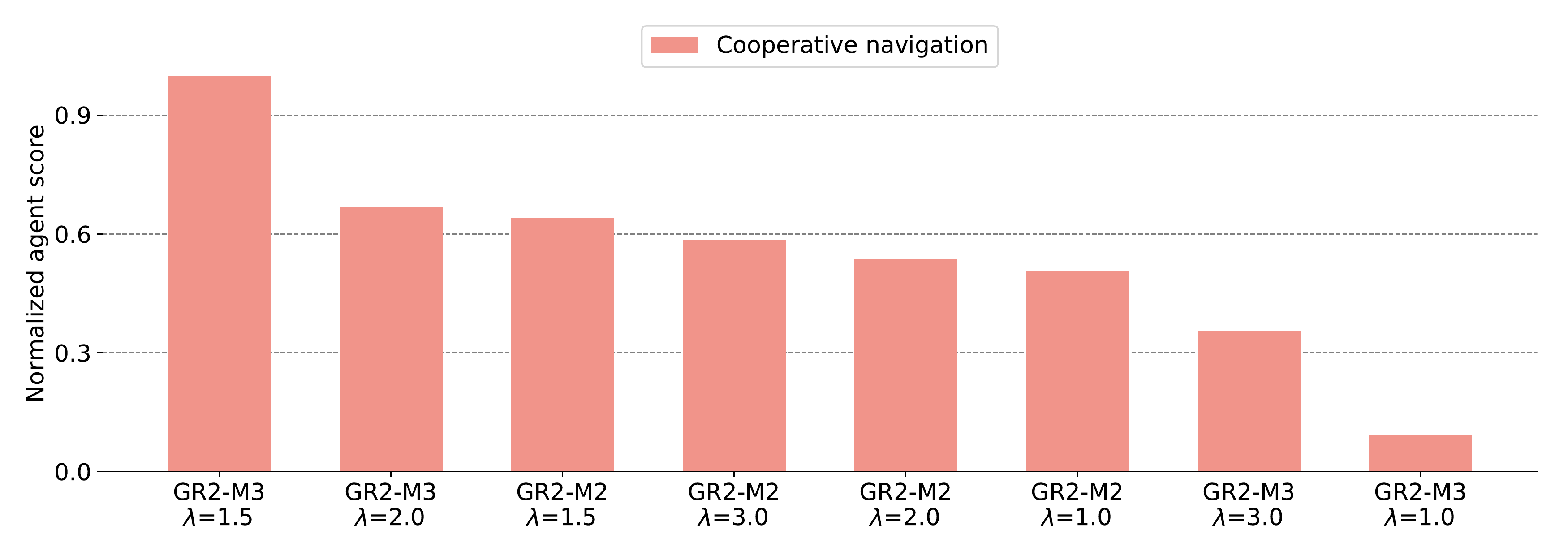}
     \caption{}
     \label{fig:coop_lambda}
     \end{subfigure}
     \begin{subfigure}[b]{.4\textwidth}
     \centering
         \includegraphics[width=1.\textwidth]{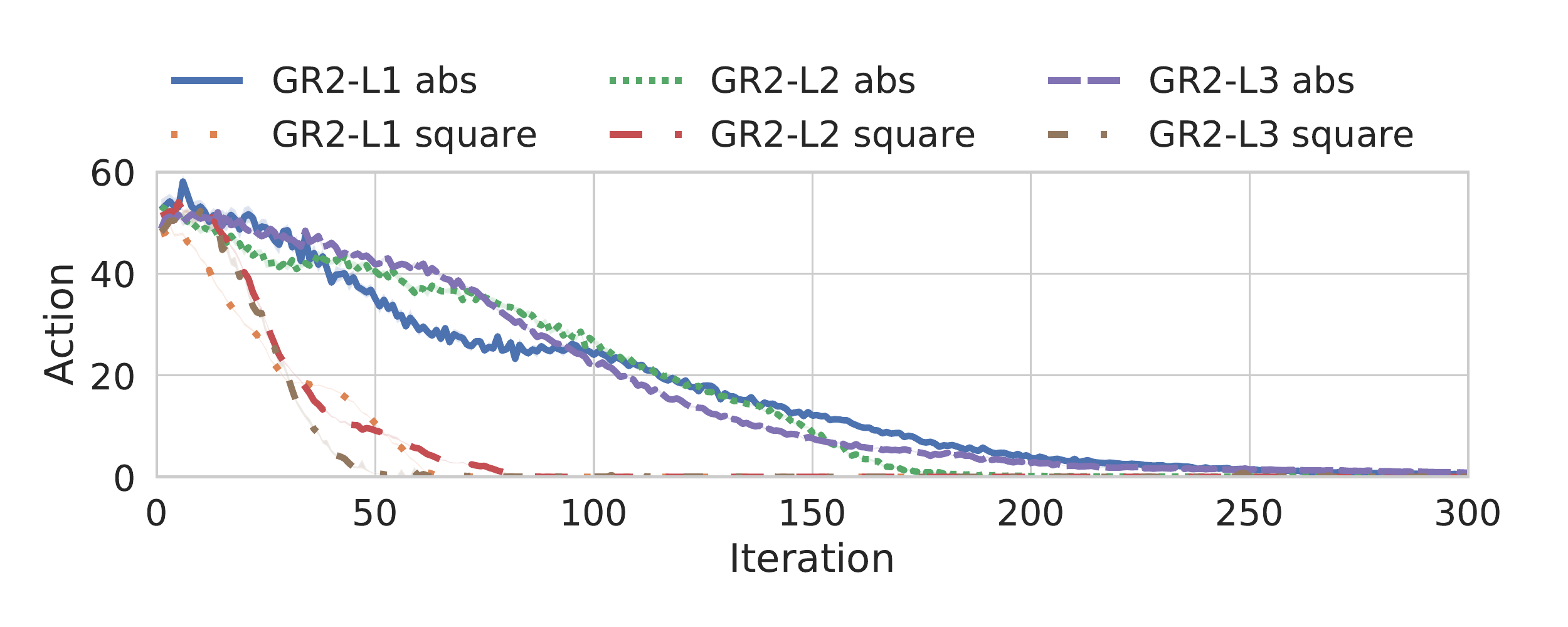}
     \caption{}
     \label{fig:pbeauty_reward}
     \end{subfigure}
     \caption{(a)Effect of varying $\lambda$ in GR2-M methods, the score is normalized to $0-1$. (b) Learning curves with two reward schemes: absolute difference (default) and squared absolute difference.}
\end{figure}

We investigate the effect of hyper-parameter $\lambda$ in the GR2-M models.
We test the  GR2-M model on the cooperative navigation game; empirically, the test selection of $\lambda=1.5$ on both GR2-M3 and GR2-M2 would lead to best performance. We therefore use $\lambda=1.5$ in the experiment section in the main paper. 

\paragraph{Choice of Reward Function in Keynes Beauty Contest.}
\label{rewards}

One sensible finding from human players  suggests that 
when prize of winning gets higher, people tend to use more steps of reasoning and they may think others will think harder too. We simulate a similar scenario by  reward shaping.  We consider two reward schemes of absolute difference and squared absolute difference. Interestingly, we find similar phenomenon in Fig.~\ref{fig:pbeauty_reward} that the amplified reward can significantly speed up the convergence for GR2-L methods.

\end{document}